\documentclass[a4paper,onecolumn,11pt]{quantumarticle}
\pdfoutput=1
\usepackage[utf8]{inputenc}
\usepackage[english]{babel}
\usepackage[T1]{fontenc}
\usepackage{amsmath}
\usepackage{hyperref}

\usepackage{tikz}
\usepackage{lipsum}

\usepackage{tikz} %

\newcommand{\citep}{\cite}
\newcommand{\citet}{\cite}

\usepackage[algo2e,linesnumbered,ruled]{algorithm2e}
\SetKwComment{Comment}{// }{}

\usepackage{amsfonts}
\usepackage{amsthm}   %
\usepackage{mathrsfs}
\usepackage{url}

\usepackage{graphicx}

\usepackage[numbers]{natbib}

\def\E{ {\mathbb{E}}}

\global\long\def\11{\mathbbm{1}}

\global\long\def\+{\oplus}

\def\<{\langle}
\def\>{\rangle}

\ifdefined \var 
  \renewcommand{\var}{\mathsf{var}}
\else
  \newcommand{\var}{\mathsf{var}}
\fi

\ifdefined \abs \else
 \newcommand{\abs}[1]{\lvert#1\rvert}
\fi

\ifdefined \norm \else
 \newcommand{\norm}[1]{\lVert#1\rVert}
\fi

\ifdefined \set 
  \renewcommand{\set}[1]{\left\{#1\right\}}
\else
  \newcommand{\set}[1]{\left\{#1\right\}}
\fi

\DeclareMathOperator*{\argmin}{arg\,min}

\DeclareMathOperator*{\tensor}{\otimes}

\providecommand{\tr}{tr}

\ifdefined \Tr 
  \renewcommand{\Tr}[1]{\tr \Big\{#1\Big\}}
\else
  \newcommand{\Tr}[1]{\tr \Big\{#1\Big\}}
\fi

\usepackage{xparse}
\usepackage{physics}

\newtheorem{definition}{Definition}
\newtheorem{remark}{Remark}
\newtheorem{theorem}{Theorem}
\newtheorem{corollary}{Corollary}

\newtheorem{example}{Example}
\newtheorem{lemma}{Lemma}

\usepackage{physics}

\def\Pr{\mathrm{P}}

\def\C{\mathbb{C}}
\def\Dist{\mathcal{D}}
\def\E{\mathbb{E}}
\def\F{\mathcal{F}}

\def\R{\mathbb{R}}
\def\union{\bigcup}

\def\union{\bigcup}
\def\intersect{\bigcap}

\def\X{\mathcal{X}}
\def\Bernoulli{\mathrm{Bernoulli}}

\def\Hyp{\mathrm{Hyp}}
\def\X{\mathcal{X}}
\def\Y{\mathcal{Y}}
\def\Alg{\mathcal{A}}

\def\DER{\mathrm{DER}}
\def\Rad{\mathfrak{R}}

\def\Part{\mathcal{P}}
\def\CNum{\mathcal{N}}
\def\fat{\mathrm{fat}}
\def\root{\mathfrak{R}}
\def\Fine{\mathfrak{F}}
\def\Out{\mathrm{Out}}
\def\Cl{\mathcal{Y}}
\def\Hilb{\mathcal{H}}
\def\Smooth{\mathbb{S}}
\def\Cplx{\mathbb{C}}
\def\GeneralVar{\mathbb{G}}

\def\JMCovNumber{\mathcal{N}_{JM}}
\def\PNum{\mathcal{M}}

\newcommand{\ceil}[1]{\lceil {#1} \rceil}

\begin{document}

\title{Fat Shattering, Joint Measurability, and PAC Learnability of POVM Hypothesis Classes}

\author{Abram Magner}
\affiliation{Department of Computer Science, University at Albany, SUNY}
\email{amagner@albany.edu}
\author{Arun Padakandla}
\affiliation{Eurecom, Nice, France}
\email{arunpr@umich.edu}
\thanks{You can use the \texttt{\textbackslash{}email}, \texttt{\textbackslash{}homepage}, and \texttt{\textbackslash{}thanks} commands to add additional information for the preceding \texttt{\textbackslash{}author}. If applicable, this can also be used to indicate that a work has previously been published in conference proceedings.}
\maketitle

\begin{abstract}

We characterize learnability for quantum measurement classes by establishing matching necessary and sufficient conditions for their PAC learnability, along with corresponding sample complexity bounds, in the setting where the learner is given access only to prepared quantum states.  
We first probe the results from previous works on this setting.  We show that the empirical risk defined in previous works and matching the definition in the classical theory can fail to satisfy the uniform convergence property enjoyed in the classical learning setting for classes that we can show to be PAC learnable. 
 Moreover, we show that VC dimension generalization upper bounds in previous work are in many cases infinite, even for measurement classes defined on a finite-dimensional Hilbert space. 
 To surmount the failure of the standard ERM to satisfy uniform convergence, we define a new learning rule -- \emph{denoised empirical risk minimization}.  We show this to be a universal learning rule for probabilistically observed concept classes, and the condition for it to satisfy uniform convergence is finite fat shattering dimension of the class. In the quantum setting, in contrast with the classical probabilistically observed case, sampled states are perturbed when a a quantum measurement is applied, according to the Born rule, so that distinct samples in the training data cannot be arbitrarily reused.  We give quantitative sample complexity upper and lower bounds for learnability in terms of finite fat-shattering dimension and a notion of approximate finite partitionability into approximately jointly measurable subsets, which allow for sample reuse.  We show that finite fat shattering dimension implies finite coverability by approximately jointly measurable subsets, leading to our matching conditions.  We also show that every measurement class defined on a finite-dimensional Hilbert space is PAC learnable. 
 We illustrate our results on several example POVM classes.

\end{abstract}

\section{Introduction}
\label{sec:introduction}

The intent of this work is to provide matching necessary and sufficient conditions for learnability in
the following supervised quantum learning\footnote{For background on classical statistical learning theory, see Appendix~\ref{sec:learning-for-quantum}.} scenario: there is an unknown joint probability distribution on prepared quantum states and classical labels.  A hypothesis class consisting of quantum measurements is fixed and known to a learner.  The learner is given access to a training dataset of these state-label pairs, but can only interact with the states by observing the classical outcomes of measuring them.  It then outputs a hypothesis that is as close as possible to minimizing the expected \emph{risk} over all hypotheses.  This learning scenario was first posed in~\cite{Heidari2021} as a quantum version of the classical PAC (Probably Approximately Correct) learning setting (see Appendix~\ref{sec:learning-for-quantum}) in which hypotheses are quantum measurements.  This setting has extensive motivations ranging from building universal quantum state discriminators to classification of unknown quantum processes to classifying quantum phases of multipartite systems (see also~\cite{MohsenQNN,QuantumKernel2023}).  The setting was then further developed in~\cite{ArunAistats}.  In contrast with more well-established quantum learning frameworks~\cite{Arunachalam2018}, which deal with quantum algorithms for learning classical hypotheses (e.g., boolean functions $f:\{0, 1\}^n \to \{0, 1\}$) from superpositions of states corresponding to classical bit strings, our framework covers a distinct scenario in which input data consists of unknown quantum states, and the goal is to learn a measurement that predicts attributes (e.g., a class label) of those states.

More specifically, the authors of~\cite{Heidari2021} formulated the quantum PAC
learning framework that we study as follows: we fix a domain $\X$ consisting of quantum states,
along with a codomain $\Cl$.  Analytically,
quantum states are described by density matrices on a fixed Hilbert space $\Hilb$
over the complex numbers $\C$.  We take the codomain $\Cl = \{0, 1\}$ for binary classification, but our results can be generalized further.
A POVM hypothesis class $\Hyp$ is a set of \emph{positive operator-valued measures}~\cite{BkWilde_2017}\footnote{See the definitions from quantum mechanics, collected in the supplementary material.}, which specify quantum measurements with outcomes in $\Cl$.  Additionally, we fix a loss function $\ell:\Cl\times\Cl\to [0, \infty]$.  For binary classification, we take the misclassification loss $\ell(y_1, y_2) = I[y_1 \neq y_2]$, where $I[\cdot]$ is the indicator function. 
The learning process is as follows: an unknown distribution $\Dist$ on $\X\times \Cl$ is fixed.  To produce a single training example, $(X, Y) \sim \Dist$ is
sampled, and then a quantum register is prepared in state $X$. \textbf{Here and throughout, a quantum register is a collection of qubits prepared in a state that is a density matrix in $\Hilb$, which may be multidimensional}.  The learner is given access to the quantum register and $Y$, and can only interact with the quantum register by measuring it and observing the outcome.
This occurs independently $m$ times to produce a training set of size $m$.  The learner is then allowed to make arbitrarily many measurements of the given quantum registers %
and by an arbitrary procedure then producing a resulting POVM $h$ from the class $\Hyp$.  We note that each measurement alters the state of the register according
to the axioms of quantum mechanics.
The risk of a hypothesis is given by
$R(h) = \E_{(X, Y) \sim \Dist}[\ell(h[X], Y)]$,
where $h[X]$ denotes a random variable whose distribution is that of the outcome of measuring a quantum register
in state $X$ with POVM $h$.  Then the goal of the learner
is to output a hypothesis with risk close enough to the minimal risk achieved by any hypothesis in the class.  We define this setup formally in Definition~\ref{def:povm-learning-problem} below.

The main problems of interest are similar to the ones asked in the classical PAC learning framework: perhaps the most immediate one is, what is a natural necessary and sufficient condition for PAC learnability of a POVM concept class?  Is there a learning rule that is universal, in the sense that it is a PAC learning rule whenever the concept class is learnable?  The present paper answers both of these questions.  In the classical case with deterministic (function) concept classes, one of the fundamental results, which
is sometimes called the \emph{fundamental theorem of concept learning}, gives a necessary and sufficient condition for learnability of a concept class for binary classification under the misclassification loss: namely, learnability is equivalent to finiteness of the Vapnik-Chervonenkis (VC) dimension of the class~\cite{UnderstandingMachineLearning}.
The recent paper~\cite{ArunAistats} gave one possible generalization of VC dimension to the quantum setting, resulting in a sufficient condition for learnability of POVM classes, along with a sample complexity\footnote{The \emph{sample complexity} of a learning rule is the minimum number of samples required to guarantee that with probability at least $1-\delta$, the risk (i.e., expected loss) of the learned hypothesis is within $\epsilon$ of the minimum possible.} upper bound for one particular learning rule.  However, it gave no necessary conditions and did not explore the tightness of the upper bound or the universality of the learning rule.  
The present paper finds that this sufficient condition is substantially weak and that the learning rule is very far from universal.
We provide a new learning rule -- \emph{denoised empirical risk minimization (DERM)}, that we can show to be universal, along with matching necessary and sufficient conditions for learnability.  See Section~\ref{sec:contributions} for a fuller list of our contributions.

\subsection{Prior work}

The literature on statistical problems involving quantum states and measurements is quite broad.  For example, a wealth of quantum state estimation problems have been posed~\cite{Arunachalam2017,arunachalam2020quantum,2021MMNatPhy_AnsAruKuwSol},
wherein the input is a sequence of multiple quantum registers, all prepared in a single unknown state. 
This set of works also includes works on state tomography~\cite{10.1145/2897518.2897544, 10.1145/3055399.3055454,
10.1145/2897518.2897585, 2006PRSMPE_Aar, Rehacek:1391369, 
10.1088/978-0-7503-3063-3ch6, Altepeter2004QubitQS, 
10.1145/3188745.3188802, 10.5555/3327546.3327572}.  The task in such studies is
to glean information about the single, unknown state -- specifically, to \emph{estimate} it.  Estimation is \emph{not} the same thing as learning, and so these are is in contrast with our work,
in which the goal is more analogous to the classical supervised learning problem: i.e., our goal is to learn a statistical association between unknown quantum states sampled according to an unknown distribution and their classical labels.  This statistical association need not reflect any intrinsic physical information about the states.  We also point out that there are various works, such as~\cite{Bshouty1998,Atici_2005} that mix what is called PAC learning with quantum information, but these differ substantially from our setting:  e.g., they assume a uniform distribution on the input, so they are not distribution-free;
or they strongly constrain the input state to correspond to a bit string; or they output a boolean function instead of a POVM.  There is also a large and expanding
body of work in quantum machine learning in which hypothesis classes consist of specially structured POVMs -- as a recent example, \cite{roget2022quantum}.
The focus in such works is different from that of the framework we study,
since they aim to solve \emph{classical} learning problems by suitably encoding
classical input data as quantum states, then choosing a suitable measurement
from the hypothesis class.  In our case, the inputs $\X$ are intrinsically quantum
and are not encodings of known classical inputs.

In the supplementary materials, %
we give a more extended discussion of prior works and how they differ from ours, including, in particular, how works on channel tomography are not applicable to solve our learning problem.

At first glance, the paper~\cite{2016QIP_CheHseYeh} has 
a more related goal to ours -- producing
an optimal POVM from training data.  However,
training samples consist of the density matrices encoding states, rather than quantum registers, as well as the probabilities
of outcomes of measurements by an unknown POVM.
In contrast, in the framework that we consider,
the inputs to a learner are not analytical state descriptions; rather, they are  quantum registers prepared in those states.  Furthermore, we are given, not probabilities of outcomes, but the outcomes themselves.  Finally, the statistical relationship between the state and the label in our case can be arbitrary, whereas in the cited paper, it is governed by a single unknown POVM.

Two recent papers are the most relevant to the present one and, indeed, are the sources of the framework that we study in this paper:~\cite{Heidari2021,ArunAistats}.
The paper~\cite{Heidari2021} formulated the POVM class PAC learning framework, showed that finite-cardinality POVM classes are PAC learnable, and pointed out the usefulness of joint measurability in reducing sample complexity, resulting in the \emph{Quantum Empirical Risk Minimization (QERM)} learning rule.  The QERM rule is a generalization of the classical ERM, which is the cornerstone of classical statistical learning theory.

The paper~\cite{ArunAistats} studied the same setting, extending the sample complexity upper bounds for the QERM rule under the assumption that a partition is given, by formulating one possible generalization of the classical VC dimension of a probabilistically observed concept class.  This implicitly showed that there exist PAC learnable POVM classes with infinite cardinality but left open the problem of giving necessary and sufficient conditions for a given class to be learnable.  For example, no necessary conditions were given, in contrast with the present work.  We will also show in this work that the upper bounds in that work are frequently vacuous.

In the course of proving our results, it will be convenient to define a PAC learning framework for what we call \emph{probabilistically observed concept classes} (POCC), which we study as a technical tool for our quantum results.  In this framework, each concept is a function from $\X$ to the set of probability distributions on $\Cl$, and the task is, as usual to learn a risk-minimizing concept.  However, on any sampled $x \in \X$ from the training set, for any concept $h$, the learner is only allowed to see a sample from the distribution $h(x)$.  We will denote such samples by $h[x]$.  This is in contrast with the theory of \emph{probabilistic concepts} ($p$-concepts) introduced in~\cite{Kearns1994}.  There, concepts are similarly conditional distributions, but the learner is allowed to see the entire distribution $h(x)$.

\subsection{Our contributions}
\label{sec:contributions}

\begin{enumerate}
    \item
        \textbf{Results on failure of ERM and uniform convergence:}
        We first show that the natural ERM learning rule proposed and studied in~\cite{Heidari2021,ArunAistats} can fail for probabilistically observed concept classes that are PAC learnable.  We probe this phenomenon further, showing that the empirical risk can fail to satisfy the uniform convergence property for learnable hypothesis classes $\Hyp$.  That is,
        the supremal deviation of the empirical risk from expected value, where the supremum ranges over all elements of $\Hyp$, does not converge to $0$ as the number of samples tends to $\infty$.
        This implies that in the probabilistically observed and the quantum case, the QERM learning rule cannot be universal in the sense of being a PAC learning rule if and only if the class to which it is applied is learnable.
        
    \item 
        \textbf{Learnability of finite dimensional hypothesis classes: }
        We then show that every POVM class defined on a finite-dimensional Hilbert space is PAC learnable.  This implies that the nontrivial \emph{qualitative} question of learnability/non-learnability only occurs in the infinite-dimensional case.  Furthermore, this implies that recovering classical learning theory from the POVM class framework requires mapping of classical classes to POVM classes over infinite-dimensional Hilbert spaces.  This is an indication that infinite-dimensional Hilbert spaces are of fundamental interest for a complete quantum learning theory.

    \item 
        \textbf{Matching conditions for learnability of infinite-dimensional hypothesis classes: }
        
        We then turn to the problem of characterizing learning in the infinite-dimensional case.
        Motivated by our result on the failure of ERM, we define a new learning rule, 
        which we call \emph{denoised empirical risk 
        minimization} (DERM).  At the heart of this is a partition of the hypothesis class into \emph{approximately jointly measurable} subsets.  Intuitively, approximate joint measurability of a set $S$ of POVMs allows us to reuse samples in the training set to evaluate the denoised empirical risk for every element of $S$.

        We show a sample complexity 
        upper bound for DERM in terms of a suitably defined version of the fat shattering dimension and the \emph{approximate joint measurability (JM) covering number} of the hypothesis class, which implies that finiteness of both of these quantities for a  POVM class is a \emph{sufficient condition} for learnability.
        
        We then exhibit a link between the JM covering number and the fat shattering dimension, thereby connecting the quantum concept with a learning theoretic complexity measure.  Specifically, we show that the JM covering number lower bounds the fat shattering dimension, which implies that finite fat shattering dimension alone is a sufficient condition for learnability.

        We next show a sample complexity \emph{lower} bound, again in terms of the fat shattering dimension.  This implies that finite fat shattering dimension is a \emph{necessary} condition for learnability.  That is, \textbf{finiteness of the fat shattering dimension is necessary and sufficient for POVM classes}.  This constitutes a fundamental theorem of concept learning for these classes.  
        While the main topic of this paper is
        the quantum setting, the result for POCC classes
        may be of independent interest, as they are the
        natural formulation of the learning problem in cases where the value of the loss function depends on unobserved variables whose joint distribution with the input $X$ is known.
        
         Our results improve substantially on prior work, which did not prove any necessary conditions and which provided sample complexity upper bounds that are frequently infinite for finite-dimensional hypothesis classes that we can show to be learnable.  

    \item 
        \textbf{Example POVM classes: }
         We then give examples of learnable and non-learnable POVM classes, including quantum neural networks (learnable) and an infinite-dimensional, nontrivally quantum example of a non-learnable class.

\end{enumerate}
 
\textbf{All proofs are provided in the supplementary material.}

\section{Main results: learnability, uniform convergence, and ERM}

\subsection{Preliminaries}
We first define the learning problems relevant to us.  Definitions from quantum mechanics can be found in~\cite{BkWilde_2017} and in the supplementary material.

\begin{definition}[POVM concept class learning problem~\cite{Heidari2021,ArunAistats}]
    \label{def:povm-learning-problem}
    In the POVM concept/hypothesis class learning problem, we 
    fix a set of possible input mixed states $\X$, which
    are density operators on a common Hilbert space
    $\Hilb$, and a set of possible classical outputs
    $\Cl$.  We fix a \emph{loss function} $\ell:\Cl \times \Cl \to [0, \infty]$.
    
    We fix a POVM concept class $\Hyp$, which is 
    simply a set of POVMs on $\Hilb$ having $|\Cl|$ outcomes.  Informally, a \textbf{learning rule} $\Alg$ in this context
    takes as input a dataset $\{(\rho_j, Y_j)\}_{j=1}^m$ consisting of quantum registers in states $\rho_j \in \X$ and classical outputs $Y_j \in \Cl$.  This dataset is sampled from an unknown joint distribution $\Dist$ on $\X\times \Cl$.  The learning
    rule interacts with the $\rho_j$ via quantum measurements (formally, POVMs). 
    Finally,
    it outputs a POVM $\Phi_* \in \Hyp$ with the goal of minimizing $\E_{(X, Y) \sim \Dist}[\ell(\Phi_*[X], Y)]$, where $\Phi_*[X] \in \Cl$ denotes the random outcome resulting from measuring $X$ with $\rho_*$.  We give a more formal definition of a learning rule in the supplementary materials. %

    We say that a POVM learning rule $\Alg$ is $(\epsilon, \delta)$-probably approximately correct (PAC) for $\Hyp$ if there exists a sample size $m = m(\epsilon, \delta)$ such that for all distributions $\Dist$ on $\X \times \Cl$ with $S \sim \Dist^{m}$, with probability at least $1-\delta$, $\Alg(S)$ outputs a hypothesis
    $h \in \Hyp$ satisfying 
    \begin{align}
        R(h) - \inf_{h_* \in \Hyp} R(h_*) \leq \epsilon.
    \end{align}
    
    We then say that $\Hyp$ is $(\epsilon, \delta)$-PAC learnable if there exists an $(\epsilon, \delta)$-PAC learning rule for $\Hyp$.  Finally, we say that
    $\Hyp$ is PAC learnable if it is $(\epsilon, \delta)$-PAC learnable for all
    $\epsilon > 0, \delta > 0$.
\end{definition}

The learning problem defined in Definition~\ref{def:povm-learning-problem} is related to the problem of \emph{probabilistically observed concept class learning}, which we introduce below.

\begin{definition}[Probabilistically observed concept class learning problem]
    \label{def:probabilistic-observed-concepts}
    
    In the probabilistically observed concept
    class (POCC) learning problem, $\X$ becomes an arbitrary set, and $\Hyp$ consists of functions $f:\X\to\Delta(\Cl)$, where $\Delta(S)$ denotes
    the set of probability distributions on a set $S$.
    
    When a hypothesis $h \in \Hyp$ is applied to an element $x \in \X$, the learning rule only observes a random sample $Z\sim h(x)$, not
    $h(x)$ itself. We denote a generic sample
    from $h(x)$ by $h[x]$. 
    
    Given this setting, the definition of PAC learning remains the same as before.
\end{definition}

\begin{remark}[Probabilistic versus probabilistically observed concept learning]
    We emphasize the important distinction between 
    the probabilistic concepts (also called 
    $p$-concepts) of~\cite{kearnsPConcepts}
    and the probabilistically observed
    concepts in the present paper: 
    in the $p$-concept framework, the output probability distribution itself is observed, rather than just a sample from it.  In our
    setting, in contrast, our learning rules are
    only allowed to see a sample from an unknown output probability distribution.
\end{remark}

\subsubsection{Connecting POVM classes with POCCs}

Here we describe the connection between the POVM 
and POCC frameworks.  The POVM framework is more 
general than the POCC one: we first show how to 
translate
the problem of learning a POCC class to one of 
learning a POVM class, along with translations of 
POCC learning rules to POVM learning rules.

Given a POCC learning problem with domain $\X$ and
hypothesis class $\Hyp$, the quantumization of
this problem is formulated as follows: we introduce
a Hilbert space $\Hilb$ with dimension equal to $|\X|$
(which may be uncountably infinite), and we choose,
arbitrarily, an orthonormal basis $B = \{e_{x}\}_{x\in \X}$.  Each $x \in \X$ corresponds to a basis element
$e_x \in \Hilb$.  The domain of the POVM learning problem
is the basis $B$.
Each hypothesis $h \in \Hyp$ bijectively
maps to a corresponding POVM $\Pi_h$ defined as follows: $\Pi_h$ first measures in the basis $B$,
uniquely identifying the input state $e_x$  with probability $1$, then postprocesses $e_x$ through the
classical channel corresponding to $h(x)$.
(We note that a POVM may be constructed by measurement
of a state with a POVM, then postprocessing the outcome through a classical channel.)

A POCC learning rule is a function of inputs $x$
and samples from an arbitrary set of hypotheses
$h[x]$.  The analogous POVM learning rule is the same
function as in the classical case, applied to the classical outcome of measurement
in the basis $B$, along with results of passing this outcome through channels associated with hypotheses
in $\Hyp$.

Thus, a POCC learning rule can be translated to a 
quantum one with exactly the same error 
characteristics.  The situation becomes more 
complicated
when we generalize to truly quantum learning settings,
because
certain operations that are possible in the classical 
case are not possible in the quantum.  In particular,
in quantum settings, the hypothesis class consists of 
non-orthogonal states, which cannot be almost surely distinguished 
from one another.  Thus, our 
learning rules cannot be functions of the inputs 
themselves, but instead can only be functions of
outcomes of measurements applied to these inputs.

\subsubsection{Jointly measurable sets of POVMs}

We also need to recall the notions of a fine-graining of a POVM and a \emph{jointly measurable} set of POVMs~\cite{Jae_2019}.
Intuitively, joint measurability will allow us to reuse samples to evaluate multiple hypotheses.
\begin{definition}[Fine-graining of a POVM]
    \label{def:fine-graining}
    Let $\Pi$ be a POVM.  We say that $\Pi$ has
    a fine-graining $(\Pi', \alpha)$, where $\Pi'$ is
    a POVM and $\alpha$ is a classical channel, if
    the outcome of $\Pi$ on any state has the same
    distribution as the outcome of $\Pi'$ on the same state, then passed through the channel $\alpha$.
    We call $\Pi'$ the \emph{root} of the fine-graining.
    
    We denote the set of fine-grainings of $\Pi$ by
    $\Fine(\Pi)$.
    
    For a given fine-graining $\Phi$, we denote its root
    by $\root(\Phi)$.
\end{definition}

\begin{definition}[Jointly measurable set of POVMs]
    \label{def:jointly-measurable}
    A set $S$ of POVMs is said to be \emph{jointly measurable} if there exists
    a POVM $\Pi_*$ such that, for every $\Pi \in S$, $\Pi$ has a fine-graining
    with root POVM equal to $\Pi_*$.  We then say that $\Pi_*$ is a root POVM for $S$.
\end{definition}

The consequence of joint measurability of $S$ is that one can obtain a sample outcome from measuring a state $\rho$ with every element of $S$ by first measuring
$\rho$ with a root POVM and then passing this outcome through each of the classical
channels corresponding to the fine-grainings of the different POVMs in $S$.

\subsection{Failure of uniform convergence and ERM for PAC learnable probabilistically observed hypothesis classes}

Having dispensed with preliminaries, we next develop our first main results.
The empirical risk minimization (ERM) rule is a cornerstone of statistical learning theory in the
setting of deterministic concept classes.  For a dataset $S = \{(X_i, Y_i)\}_{i=1}^m$, the empirical risk of a hypothesis $h\in \Hyp$ is given by
\begin{align}
    \hat{R}(h, S)
    = \frac{1}{m} \sum_{j=1}^m \ell(h[X_i], Y_i).
\end{align}
In the deterministic case, a hypothesis class $\Hyp$ being PAC learnable is logically equivalent to ERM
being a PAC learning rule, which is logically equivalent to
it satisfying the following uniform convergence
property: for any hypothesis $h \in \Hyp$ and any data-generating distribution
$\Dist$, 
$ %
    \Pr_{S\sim \Dist}[ | R(h) - \hat{R}(h, S)| \geq \epsilon] \leq \delta.
$ %

The ERM rule has been proposed for use as a 
subroutine in the quantum setting in prior 
work~\cite{Heidari2021} and also adopted in the more recent work~\cite{ArunAistats}.
Both of these works give sample complexity upper bounds for this ERM rule.
Our first main result is 
that uniform convergence and the ERM rule can fail 
for a POCC class $\Hyp$, despite $\Hyp$ being PAC 
learnable.  This is in stark contrast to the deterministic case.  We will see,
in our Theorem~\ref{thm:failure-unif-convergence-povm}, that this has further implications for the quantum setting, and thus for the tightness of the bounds
in~\cite{ArunAistats}.

\begin{theorem}[Failure of uniform convergence and ERM for POCC classes]
    \label{thm:failure-unif-convergence-pocc}
    There exists a POCC class $\Hyp$ that is PAC
    learnable but for which the ERM rule is not PAC and does not satisfy the uniform convergence property.
    
    Furthermore, there exists a POCC class $\widehat{\Hyp}$
    and a choice of $\X, \Cl$, and $\Dist$ for which
    the uniform convergence property is not satisfied,
    but the ERM rule is PAC.
\end{theorem}

\subsection{Failure of uniform convergence for most finite-dimensional POVM classes}

We next show that the situation regarding ERM is even worse in the quantum case.  In particular, a consequence of what we show next is that the sample complexity upper bounds in~\cite{ArunAistats} are infinite (i.e., vacuous) for a very large class of POVM classes that are learnable.  To do
so, we recall the definition of a deterministic POVM.

\begin{definition}[Deterministic POVM]
    A POVM $\Pi = \{\Pi_0, \Pi_1\}$ is deterministic if either $\Pi_0 = 0$ or $\Pi_1 = 0$.
\end{definition}
That is, the outcome of a deterministic POVM is the same when used to measure any state.  We note that if a POVM is not deterministic, then its outcome is statistically dependent on the state that it is used to measure.

We also define the $L_1$ operator norm for operators on a Hilbert space
$\Hilb$.  For an operator $\Gamma:\Hilb\to\Hilb$, the $L_1$ operator norm is given by
\begin{align}
    \| \Gamma \|_{op,L_1}
    = \sup_{x \in \Hilb} \frac{ \|\Gamma x\|_{1} }{\| x \|_{1}}.
\end{align}
Any norm generates a topology, which allows us to talk about open and closed sets.

\begin{theorem}[Failure of uniform convergence of ERM for most finite-dimensional POVM classes]
    \label{thm:failure-unif-convergence-povm}
    Let $\X$ be a subset of a finite-dimensional Hilbert space $\Hilb$.  
    Consider an $L_1$ operator norm-closed POVM hypothesis class $\Hyp$ satisfying
    the following conditions: 
    \begin{enumerate}%
        \item 
            $\Hyp$ is jointly measurable.
        \item
            $\Hyp$ has infinite cardinality.
    \end{enumerate}
    Then exactly one of the following conclusions holds:
    \begin{enumerate}
        \item
            Uniform convergence for ERM does not hold for $\Hyp$, and ERM is not PAC.
        \item
            The only points of accumulation of $\Hyp$
            are deterministic POVMs.
    \end{enumerate}
\end{theorem}

Theorem~\ref{thm:failure-unif-convergence-povm} effectively says that an infinite-cardinality (but possibly finite-dimensional) POVM 
class can only enjoy the uniform convergence property 
for ERM if it ``clusters'' around deterministic 
measurements.  The only deterministic measurements are the ones whose 
outcomes do not depend on the states being measured.  
This implies that ERM is not a useful learning rule for
a rich enough set of POVM classes.  Since ERM was a core subroutine of~\cite{ArunAistats}, this provides useful insight on prior work: in particular,
in that work, a sample complexity upper bound for the ERM rule is given in the case where one can find a finite-cardinality jointly measurable partition.  Our theorem above implies that this upper bound must be $\infty$ unless almost all of the hypotheses are close to deterministic (and, thus, independent of the input state).  In our subsequent theorems, we will show that the upper bound of infinity is, in infinitely many cases, hopelessly loose.

\section{Main results: Every finite-dimensional POVM class is learnable}
\label{sec:finite-dimension-learnable}

In this section, we give a complete characterization of learnability of POVM classes in the case where $\X$ is a (\textbf{possibly infinite-cardinality}) subset of the set of density operators on a finite-dimensional Hilbert space $\Hilb$.  We call a POVM class defined on $\X$ a finite-dimensional POVM class.  It turns out that \emph{every finite-dimensional POVM class is learnable} -- Theorem~\ref{thm:finite-dimensional-learnable}.

\begin{theorem}[Every finite-dimensional POVM class is learnable]
    \label{thm:finite-dimensional-learnable}
    Let the span of  the domain $\X$ be a finite-dimensional subspace of the space of density operators on a Hilbert space $\Hilb$.
    Let $\Hyp$ be a POVM class all of whose POVMs are defined on $\X$.
    
    Then $\Hyp$ is PAC learnable with the following sample complexity:
    \begin{align}
        n_{\Hyp}(\epsilon, \delta)
        \leq \sum_{r=1}^N \frac{8}{\epsilon^2}\log\frac{ 2N }{\delta}
        = \frac{8N}{\epsilon^2} \log\frac{2N}{\delta},
    \end{align}
    where $N$ is the $\epsilon/4$-total variation covering number of $\Hyp$,
    which is finite.
\end{theorem}
We note that in the worst case, the covering number in Theorem~\ref{thm:finite-dimensional-learnable} can be exponential in the dimension of the Hilbert space.  However, hypothesis classes of interest, where the POVMs have constrained structure, have a much smaller covering number.  Additionally, in certain cases, one can take advantage of joint measurability in order to tighten this bound.

The above theorem provides infinitely many examples
of POVM classes that are learnable.  Furthermore, this class of examples includes ones such that the sample complexity upper bounds given in~\cite{ArunAistats}
were infinite.  Therefore, this is a substantial
improvement on the previous results. 

Interestingly, the proof involves concocting a learning rule that uses ERM, but in a different way from prior work.  This approach only works for the finite-dimensional case, necessitating yet another learning rule for our subsequent results.

\section{Main results: Matching necessary and sufficient conditions for infinite-dimensional POVM class learnability}

The result in Theorem~\ref{thm:finite-dimensional-learnable} leaves open the questions of
necessary and sufficient conditions for infinite-dimensional POVM classes and POCC classes.
Furthermore, the results in Theorems~\ref{thm:failure-unif-convergence-pocc} and ~\ref{thm:failure-unif-convergence-povm} motivate a search for an alternative learning rule to ERM for both the POCC and POVM cases.

In Section~\ref{sec:denoised-erm}, we present our
learning rule -- the \emph{denoised ERM} -- for POVM classes and how it specializes to the POCC case.  
We then show in Section~\ref{sec:necessary-and-sufficient} necessary and sufficient conditions for
PAC learnability of POCC and POVM classes.  Specifically, Theorem~\ref{thm:pac-povm-necessary-and-sufficient} gives distinct necessary and sufficient conditions.  In Theorem~\ref{thm:fat-shattering-joint-measurability}, we show an inequality relating a quantity called the \emph{approximate joint measurability covering number} to the \emph{fat shattering dimension} of $\Hyp$, which allows us to conclude with Corollary~\ref{corollary:matching-conditions} that the conditions in Theorem~\ref{thm:pac-povm-necessary-and-sufficient} are matching.

\subsection{Rescuing ERM: De-noised empirical risk}
\label{sec:denoised-erm}

We now turn to the definition of our new learning rule, called \emph{denoised empirical risk minimization}.
We first define the denoised empirical risk, which is for a hypothesis in a set of jointly measurable
POVMs.

\begin{definition}[Denoised empirical risk of a hypothesis]
    Let $H$ be a jointly measurable set of POVMs
    with a fine-graining $(\Pi, \{\alpha_{h}\}_{h\in H})$.
    
    Let $h \in H$ be a hypothesis, and let $S = \{ (X_j, Y_j) \}_{j=1}^m \in (\X\times \Cl)^m$ be a dataset.  Let $Z_j$ denote the random outcome of
    measurement of $X_j$ with the POVM $\Pi$.
    We define the \emph{denoised empirical risk of $h$ on input $S$} to be
    \begin{align}
        \DER(h, S)
        = \frac{1}{m} \sum_{j=1}^m \E[ \ell(h[X_j], Y_j) ~|~ \{ (Z_j, Y_j) \}_{j=1}^m ].
    \end{align}
    Note that the denoised empirical risk is a random
    variable.
\end{definition}

\begin{remark}
    It is \textbf{essential} to note in this definition \emph{what} is being conditioned on.  Specifically, one can imagine definitions of the empirical risk that either suffer
    from the same flaws as the ordinary empirical risk or, alternatively, cannot
    be computed by the learner because of the learner's inability to see the $X_j$.
    Intuitively, our definition ``averages out'' the randomness from the classical
    channels, which mitigates the drawbacks of the ordinary empirical risk.
\end{remark}

To state the learning rule, we also need a relaxation of the notion of a jointly measurable
partition introduced in~\cite{ArunAistats}.  
To state this, we also need to define the total variation distance between POVMs.

\begin{definition}[Total variation distance between POVMs]
    \label{def:dtv-povms}
    Let $\Pi_1, \Pi_2$ be two POVMs with common domain $\X$.  We define the total variation distance between $\Pi_1, \Pi_2$ as follows:
    \begin{align}
        d_{TV}(\Pi_1, \Pi_2)
        = \sup_{x \in \X} d_{TV}(\Out(\Pi_1, x), \Out(\Pi_2, x)),
    \end{align}
    where $\Out(\Pi, x)$ denotes the random outcome
    of the POVM $\Pi$ on the mixed state $x$.
\end{definition}

With this in hand, we next define an \emph{approximately jointly measurable class}.

\begin{definition}[Approximately jointly measurable class]
    We say that a collection $S$ of POVMs is $\gamma$-approximately jointly measurable (or just $\gamma$-jointly measurable) if there exists 
    a POVM $\Pi_*$ such that, for every POVM $\Pi \in S$, $\Pi$ has a fine-graining with root $\Pi'$
    satisfying
    $ %
        d_{TV}(\Pi_*, \Pi') \leq \gamma,
    $ %
    where we recall the definition of the total
    variation distance between two POVMs in Definition~\ref{def:dtv-povms}.  We call
    $\Pi_*$ a \emph{center} of $S$.  
\end{definition}

Our proofs will exploit the fact that an approximately jointly measurable class $S$ can be approximated by a jointly measurable one by choosing a center POVM $\Pi_*$ of $S$ and replacing each element of $S$ with a POVM whose root is $\Pi_*$.  This is formalized in the following definition.
\begin{definition}[Joint measurability-smoothed class]
    Let $S$ be a $\gamma$-jointly measurable POVM class, and let $\Pi_*$ be a center for
    $S$.  
    For a POVM $h \in S$, we define $\Smooth_{JM}(h, \Pi_*)$ to be 
    the following POVM:
    \begin{align}
        \Smooth_{JM}(h, \Pi_*)
        = \argmin_{\Pi ~:~ \Pi_* \in \root(\Fine(\Pi))} d_{TV}(\Pi, h).
    \end{align}   
    
    We denote by $\Smooth_{JM}(S, \Pi_*)$ the following set of POVMs:
    \begin{align}
        \Smooth_{JM}(S, \Pi_*)
        = \{ \Smooth_{JM}(h, \Pi_*) ~|~ h \in S \}.
    \end{align}
\end{definition}

We next define joint measurability notions for general hypothesis classes.

\begin{definition}[Approximately jointly measurable partition]
    Let $\Hyp$ be a POVM class. 
    Then a partition $\Part$ of $\Hyp$ is a $\gamma$-approximately jointly measurable partition of $\Hyp$ if each partition element $P_j$ is $\gamma$-approximately jointly measurable.
\end{definition}

\begin{definition}[Approximate joint measurability covering numbers]
    label{def:joint-measurability-covering}
    Let $S$ be a collection of POVMs.  We say that a collection of subsets
    $\{S_j\}_{j\in\Omega}$ of $S$ 
    is a $\gamma$-joint measurability covering of $S$ if $\union_{j \in \Omega} S_j = S$ and every $S_j$ is $\gamma$-jointly measurable.

    We define the $\gamma$-joint measurability covering number $\JMCovNumber(\gamma, S)$ to be the minimum cardinality of a $\gamma$-joint measurability covering of $S$.
\end{definition}

Having defined the denoised empirical risk of a 
hypothesis and the notion of an approximately jointly measurable partition of a hypothesis class, we present the denoised empirical risk 
minimization rule in Algorithm~\ref{alg:DERM}.  This
rule is parametrized by an approximately-jointly measurable partition with finite cardinality and a choice, for each partition element $P_j$, of a number of samples
$n_j$, satisfying $\sum_{j=1}^R n_j = m$, where
$m$ is the size of the input training set.  We will explain how to choose $n_j$ based on the sample complexity bounds that we will derive.  The input training set is partitioned into consecutive subsets
$\hat{S}_1, ..., \hat{S}_R$, of cardinalities
$n_1, ..., n_R$.

\begin{algorithm2e}[th]
    \caption{Denoised empirical risk minimization learning rule 
    \label{alg:DERM}}
    \KwData{ Training set $S = \{ (X_j, Y_j) \}_{j=1}^m$;  Approximately jointly measurable partition $\Part = \{ P_j \}_{j=1}^{ R}$ of $\Hyp$ with centers $\Pi_*^{(j)}$; partitioned training set $\{\hat{S}_i\}_{i=1}^R$ }
    \KwResult{ A hypothesis $h_* \in \Hyp$ meant to have almost minimum risk } 
    
    \For{$j = 1$ to $R$}{
        \Comment{Process partition element $j$.}
        Let $\hat{P}_j = \Smooth(P_j, \Pi_*^{(j)})$\;
        Compute $\hat{\rho}_j = \argmin_{h \in \hat{P}_j} \DER(h, \hat{S}_j)$ 
        and $\hat{R}_j = \DER(\hat{\rho}_j, \hat{S}_j)$\;
    }
    Let $j_* = \argmin_{j\in \{1, ..., R\}} \hat{R}_j$\;
    
    Let $h_*$ be some hypothesis in $\Hyp$ such that $\Smooth(h_*, \Pi_*^{(j)}) = \hat{\rho}_{j_*}$\;
    
    \Return{$h_*$}\; %
\end{algorithm2e}
\begin{remark}[The distinction between learning rules and algorithms]
    A distinction is made in statistical learning theory between a learning rule and an algorithm.
    In the classical case, ERM is a learning rule, \emph{not} an algorithm, because it does not specify
    \emph{how} to achieve the minimization.  Achieving
    the minimization algorithmically requires more 
    refined structural knowledge of the hypotheses.
    Indeed, for many hypothesis classes, ERM is a 
    PAC learning rule, but it is not efficiently
    implementable by an algorithm.  The study
    of efficient algorithmic implementation of
    learning rules is the subject of computational,
    rather than statistical, learning theory.
    
    In our case, the DERM is similarly a learning rule, not an algorithm.  We do not, and
    cannot, prescribe how to construct the jointly measurable partition, for instance.  It is only important that the partition exists.
\end{remark}

\subsection{Necessary and sufficient conditions for PAC learnability}
\label{sec:necessary-and-sufficient}

We next turn to necessary and sufficient conditions
for PAC learnability of POVM and POCC classes.  To
do this, we need to define the fat shattering dimension
of a class of POVMs~\cite{ChervonenkisFestschrift}.  We start
by recalling the fat shattering dimension in the deterministic hypothesis class case.

\begin{definition}[Fat-shattering dimension (classical case)]
    Let $\F$ be a class of functions $f:\X\to\R$.
    We say that a dataset $(x_1, ..., x_m) \in \X^m$ is $\gamma$-fat-shattered by $\F$ if
    there exist \emph{witness numbers} $r_1, ..., r_m$ such that, for any $(b_1, ..., b_m) \in \{ 0, 1\}^m$, there exists some $f \in \F$ such that,
    simultaneously, $f(x_i) \leq r_i - \gamma$ if
    $b_i = 0$ and $f(x_i) \geq r_i + \gamma$ if
    $b_i = 1$.
    
    We define the $\gamma$-fat-shattering dimension 
    of $\F$ to be the largest $m$ for which there
    exists a dataset $X \in \X^m$ that is $\gamma$-fat-shattered by $\F$.
\end{definition}

To define an analogous notion for POVM classes $\Hyp$, we identify each POVM $h$ with its induced function mapping from $\X$ to probability distributions on $\{0, 1\}$, which themselves
can be identified with real numbers in $[0, 1]$.  Thus, each hypothesis induces a function
$f_h:\X\to[0, 1]$, and we define the $\gamma$-fat-shattering dimension of $\Hyp$ to be that
of the induced deterministic function hypothesis class.

\begin{theorem}[Necessary and sufficient conditions for PAC learnability of POVM classes]
    \label{thm:pac-povm-necessary-and-sufficient}
    Let $\Hyp$ be a POVM class (not necessarily finite-dimensional or finite-cardinality).  Then the following statements hold.
    
    \begin{enumerate}
        \item
            \textbf{Necessary condition: }
            If $\Hyp$ is PAC learnable, then for every $\gamma > 0$, $\Hyp$ has
            finite $\gamma$-fat-shattering dimension.

        \item
            \textbf{Sufficient condition: }
            If there exists a finite $d > 0$ such that, for every small enough $\gamma > 0$, $\Hyp$ has $\gamma$-fat-shattering dimension $\leq d$
            and, additionally,
            for every $\alpha > 0$, there exists a finite $\alpha$-almost-jointly-measurable partition $\Part$ of $\Hyp$, then $\Hyp$ is PAC learnable by DERM, with sample complexity
            \begin{align}
                n_{\Hyp}(\epsilon, \delta)
                \leq \inf_{\Part}  O(|\Part|\frac{d + \log(1/\delta)}{\epsilon^2}),
            \end{align}
            where $\Part$ ranges over all finite $\epsilon/8$-almost-jointly-measurable partitions of $\Hyp$.
    \end{enumerate}
\end{theorem}

We next present a theorem linking fat-shattering dimension to approximate joint measurability.

\begin{theorem}[Lower bound on fat-shattering dimension via approximate joint measurability covering number]
    \label{thm:fat-shattering-joint-measurability}
    Let $\Hyp$ be a POVM class with $k \leq \JMCovNumber(\gamma, \Hyp)$ and with $d = \fat_{\gamma/4}(\Hyp)$.  Then for every $\gamma > 0$, if $m \geq {k\choose 2}$, then
    \begin{align}
        k 
        \leq
        2\cdot (m\cdot (2/\gamma + 1)^2)^{\ceil{ d \cdot \log\left( \frac{2em}{d\gamma} \right) }}.
    \end{align}
    In particular, if $\JMCovNumber(\gamma, \Hyp) = \infty$, then $\fat_{\gamma/4}(\Hyp) = \infty$.
\end{theorem}

Theorem~\ref{thm:fat-shattering-joint-measurability} immediately implies the following
corollary of Theorem~\ref{thm:pac-povm-necessary-and-sufficient}.

\begin{corollary}[Fundamental theorem of concept learning for POVM classes]
    \label{corollary:matching-conditions}
    Let $\Hyp$ be a POVM class.  Then
    $\Hyp$ is PAC learnable if and only if, for every $\gamma > 0$, $\Hyp$ 
    has finite $\gamma$-fat-shattering dimension.  Furthermore, DERM is a PAC learning rule
    for $\Hyp$.  That is, DERM is a universal learning rule.
\end{corollary}

Specifically, this is because finite fat-shattering dimension implies finiteness of $\JMCovNumber(\gamma, \Hyp)$ for all $\gamma$, rendering the finiteness of $\JMCovNumber(\gamma, \Hyp)$ redundant as a sufficient condition for learnability.

Corollary~\ref{corollary:matching-conditions} constitutes a fundamental theorem of concept learning for POVM classes.
As mentioned in the introduction, this implies the same for POCC classes, which
arise whenever the loss function value depends on unobserved variables.

\subsection{Examples and applications}

Here we present example POVM classes to illustrate our results.
We first present an application of our results to the learnability of variational quantum circuits (sometimes called quantum neural networks).

\begin{definition}[Quantum neural network~\cite{MohsenQNN}]
    Consider a Hilbert space $\Hilb$ with dimension
    $d < \infty$.  An $\ell$-layer quantum neural network is
    parametrized by a sequence of $\ell$ unitary
    operators $U_{1}, U_{2}, ..., U_{\ell}$ and
    a measurement $\Pi$, all mapping $\Hilb$ to $\Hilb$,
    with $\Pi$ having two classical outcomes.
    Each of the unitary operators $U_{j}$ acts
    nontrivially on only a subset of qubits.
\end{definition}

\begin{theorem}[Learnability of quantum neural networks]
    For any $d < \infty$, the class of quantum
    neural networks is PAC learnable.
\end{theorem}
\textbf{We give much more detailed results on learning variational quantum circuits
in the supplementary material.}  %

As in classical learning theory, one advantage of combinatorial bounds on generalization error, such as those in terms of VC or fat-shattering dimension, is that they are \emph{distribution-free}, meaning that they do not depend on the data generating distribution.  This generality, of course, comes at a cost of tightness, as is well-known to be the case for neural networks.  We have nonetheless included the above example to illustrate the application of our bounds to a concrete hypothesis class.  We emphasize that our bounds can be applied beyond classes of quantum neural networks, and that our focus in this work is on generality, as is the case in the classical works using the fat-shattering dimension.

\paragraph{Non-learnable POVM classes:}

We next turn
to examples of POVM classes that are \emph{not} learnable.  Theorem~\ref{thm:finite-dimensional-learnable}
implies that we must consider input state sets spanning
an infinite-dimensional Hilbert space.  One can
easily concoct an unlearnable class from a classical
hypothesis class with infinite VC dimension.  Our
next example goes further than this to provide intuition
about legitimately quantum POVM classes that are unlearnable.

\begin{example}[A POVM class that is not learnable]
    Consider a domain $\X$ whose span is infinite-dimensional in some Hilbert space $H$
    with a sequence of orthonormal vectors $\{\rho_j\}_{j=1}^\infty$ and a sequence
    of numbers $\{\beta_j\}_{j=1}^\infty$, with $\beta_j \in (1/2 + \gamma, 1)$.  We construct
    a sequence of hypotheses as follows (since each hypothesis has two outcomes, we only
    need to specify for each hypothesis the operator corresponding to the $0$ outcome): letting
    $b = (b_1, b_2, ...) \in \{0, 1\}^{\infty}$,
    \begin{align}
        \Pi_0^{(b)} = \sum_{j=1}^\infty \beta_j^{1-b_j}(1 - \beta_j)^{b_j} \ket{\rho_j}\bra{\rho_j}.
    \end{align}
    
    Finally, let $\Hyp$ be the class of all such POVMs:
    $ %
        \Hyp = \{ \Pi^{(b)} \}_{b \in \{0, 1\}^{\infty}}.
    $ %
    We claim that $\Hyp$ is not PAC learnable.  To show this, by Theorem~\ref{thm:pac-povm-necessary-and-sufficient}, it is sufficient to show that
    $\Hyp$ has infinite $\gamma$-fat shattering dimension.  We show this in the following sequence of steps:
    \begin{enumerate}
        \item
            We exhibit a candidate sequence $x_1, x_2, ...$ of elements of
            $\X$ that we will show to be fat-shattered by hypotheses in $\Hyp$.
            Specifically, we take
            $x_j = \ket{\rho_j}\bra{\rho_j}$.
        \item
            We exhibit a \emph{witness number} $r = 1/2$.
        \item
            We exhibit, for each bit string
            $b = b_1, b_2, ... \in \{0, 1\}^{\infty}$, 
            a hypothesis $h \in \Hyp$ such that
            $\Pr[ h[x_j] = b_j] \geq r + \gamma$
            for every $j$.
            Specifically, we take 
            $h = \Pi^{(b)}$.  We have
            \begin{align}
                &\Pr[ h[x_j] = 0] 
                = \Tr(x_j \Pi^{(b)}) \\
                &= \bra{\rho_j} \Pi^{(b)} \ket{\rho_j}
                = \beta_j^{1-b_j}(1-\beta_j)^{b_j}, 
            \end{align}
            so that
            $ %
                \Pr[ h[x_j] = b_j] = \beta_j
                \geq 1/2 + \gamma,
            $ %
            by assumption.
    \end{enumerate}
    Thus, $\Hyp$ has infinite $\gamma$-fat-shattering
    dimension, which implies that it is not PAC learnable.
\end{example}
We note that the results of prior papers were incapable of showing that \emph{any} POVM class is unlearnable.

\section{Conclusion}

We have provided matching necessary and sufficient conditions for learnability of POVM hypothesis classes in terms of their fat-shattering dimension.  To do so, we connected the learning-theoretic notion of fat-shattering dimension with the quantum concept of approximate joint measurability covering.  The proof of our sufficient condition came via the introduction of a new universal learning rule, the de-noised empirical risk minimization rule.  Additionally, we showed that all finite-dimensional POVM classes are learnable, and we provided quantitative sample complexity bounds for some example hypothesis classes.

There are various possible extensions of our work: for instance, a characterization of the fat-shattering dimension of a hypothesis class in terms of its Hilbert space geometry would be of interest.  
Additionally, our learning rule only makes \emph{separable} measurements.  In quantum hypothesis testing, where the goal is to distinguish between two \emph{known} states with minimal error probability from $m$ copies of one of them, block measurements have a provable advantage in terms of sample complexity.  It would be interesting to understand whether this phenomenon holds in the learning setting.

\section{Acknowledgments}
This research is funded by NSF CCF grant number \#2212327.

\bibliographystyle{quantum}
\bibliography{QuantumLearning,references}

\section{Supplementary material: proofs}

\subsection{Proof of Theorem~\ref{thm:failure-unif-convergence-pocc}}

    Consider $\X = \{0, 1\}^{2}, \Y = \{0, 1\}$,
    so that each $x \in \X$ can be written as $(x_1, x_2)$.
    Furthermore, consider the following hypothesis
    classes  $\widehat{\Hyp}$ and $\Hyp$: fix some small enough $\alpha > 0$, and let $\widehat{\Hyp}$ consist of probabilistically observable concepts that, on input $x \in \X$, pass $x_1$ through a binary symmetric
    channel with crossover probability $0.1 \pm z$,
    where $z \in [-\alpha, \alpha]$.  We then define
    $\Hyp$ to consist of $\widehat{\Hyp}$, with one additional hypothesis: $h_*$, which, on
    input $x$, outputs a $\Bernoulli(1/2)$ random variable with probability $0.99$ and outputs $x_2$ with probability $0.01$.
    
    One can think of $\widehat{\Hyp}$ as a relaxation of the very 
    simple deterministic hypothesis class $\Hyp'$ 
    consisting of a single hypothesis: $f(x) = x_1$.
    Of course, $\Hyp'$ is agnostic-PAC learnable.

    The theorem statement consists of the following claims:
    \begin{enumerate}
        \item
            The uniform convergence property for ERM fails to hold for the hypothesis class $\widehat{\Hyp}$, even for distributions on $\X\times \Cl$ for
            which ERM is PAC.
        \item
            The hypothesis class $\Hyp$ is PAC learnable, but there exist distributions
            for which, simultaneously, the uniform convergence property fails to hold for ERM and ERM is not PAC.
    \end{enumerate}
    
    \textbf{Proof of claim 1: }
    To show that the class $\widehat{\Hyp}$ does not satisfy the uniform convergence property, we will exhibit a data-generating distribution $\Dist$ on $\X\times\Y$ for which, with non-negligible probability, there exist hypotheses in $\widehat{\Hyp}$ whose empirical risks are bounded away from their true risks.
    In particular, let us consider a uniform distribution on $\X$, and a target function
    $f(x) = x_2$.  
    Call the resulting joint distribution $\Dist$.
    Note that the expected risk $\E[R(h, f)] = 1/2$ for all $h \in \widehat{\Hyp}$.  In 
    particular, this implies that we can trivially find a hypothesis whose expected risk is
    arbitrarily close to the minimum possible with probability exactly $1$, so this distribution does not pose any fundamental difficulties from a learning perspective.
    Next, we show that for every $m$, with probability $1$, there exists some hypothesis $h \in \widehat{Hyp}$ whose
    empirical risk on a dataset of length $m$ is bounded away from its expected risk.  Fix $m$
    samples $S = (x^{(1)}, ..., x^{(m)})$ drawn iid from $\Dist$.  We claim that with probability $1$,
    the set of \emph{outputs} of all hypotheses in $\Hyp$ on input $S$ has cardinality $2^{m}$, so
    that it is the set of bit strings of length $m$.  To prove this, we upper bound the probability of the negation of this event: let $B$ be the event that there exists some bit string that is \emph{not} the output of any hypothesis in $\Hyp$:
    \begin{align}
        \Pr[B]
        = \Pr[ \union{y \in \{0, 1\}^m} \intersect_{h \in \Hyp} [h(S) \neq y]] 
        \leq \sum_{y \in \{0, 1\}^m} \Pr[ \intersect_{h \in \Hyp} [h(S) \neq y]].
    \end{align}
    Each term of the remaining sum can be computed by conditioning on the value of $S$.  That is,
    \begin{align}
        \Pr[ \intersect_{h \in \Hyp} [h(S) \neq y]] 
        &= \E[ \Pr[ \intersect_{h \in \Hyp} [h(S) \neq y] ~|~ S] ] 
        = \E[ \prod_{h \in \Hyp} \Pr[ h(S) \neq y ~|~ S]    ]  \\
        &= \frac{1}{2^{nm}} \sum_{\hat{S} \in \{0, 1\}^{n\times m}}  \prod_{h \in \Hyp} \Pr[h(S) \neq y ~|~ S = \hat{S}].
    \end{align}
    Now, note that $\Pr[h(S) \neq y ~|~ S = \hat{S}] \leq c < 1$, by our choice of hypothesis class.  This implies that
    \begin{align}
        \Pr[ \intersect_{h \in \Hyp} [h(S) \neq y]]
        = \frac{1}{2^{nm}} \cdot 2^{nm} \prod_{h \in \Hyp} c
        \leq \prod_{h \in \Hyp} c
        = 0,
    \end{align}
    where the last equality is by the fact that $\widehat{\Hyp}$ has infinite cardinality.  This implies that
    \begin{align}
        \Pr[B]
        \leq \sum_{y \in \{0, 1\}^m} 0
        = 0.
    \end{align}
    Thus, with probability $1$, the set of outputs of all hypotheses in $\widehat{\Hyp}$ in input $S$ has 
    cardinality $2^{m}$.  Now, this means that with probability $1$, there exists a hypothesis $h$ 
    with exactly $0$ misclassification error on $S$, so that its empirical risk is $0$, while its expected risk is $1/2$.
    
    Thus, \textbf{the uniform convergence property does not hold for $\widehat{\Hyp}$}, and we have 
    demonstrated this with a distribution on which empirical risk minimization trivially outputs a
    good hypothesis.  This implies that uniform convergence is not necessary for ERM to be PAC.  This completes the proof
    of Claim~1.

    \textbf{Proof of Claim 2: }
    To show that there exist data-generating distributions for which ERM is not PAC for $\Hyp$ and the uniform convergence property fails to hold,
    we will consider the same input distribution and 
    target function as in the proof of Claim 1: then the expected risk of $h_*$ is as follows:
    \begin{align}
        \E[R(h_*, f)]
        = \Pr[ h_*(x) \neq f(x) ] 
        = \Pr[ h_*(x) \neq x_2] 
        = 0.99 \cdot 1/2.
    \end{align}
    That is, $h_*$ is the unique hypothesis in 
    $\widehat{\Hyp}$ with minimum expected risk for this 
    distribution.  By the analysis of $\widehat{\Hyp}$ above, 
    though, ERM fails to return
    $h_*$ on this hypothesis class asymptotically almost surely as the number of samples tends to $\infty$, because, with probability $1$ over the choice 
    of samples, there exists some \emph{other} hypothesis that
    has empirical risk $0$.  Thus, ERM fails in this case -- it returns a hypothesis with expected risk strictly bounded away from the minimum possible with probability $1$.

    To complete the proof of the claim, and, hence, the proof of Theorem~\ref{thm:failure-unif-convergence-pocc},
    we need to show that $\Hyp$ is PAC learnable.  This can be done in a variety of ways.  We consider
    the learning rule that works as follows: on input
    $S = \{(x_1, y_1), (x_2, y_2), ..., (x_m, y_m)\}$, 
    we perform the following operations:
    \begin{enumerate}
        \item
            Estimate the joint distribution $\Dist$
            empirically:
            \begin{align}
                \hat{p}(x, y)
                = \frac{| \{ j ~:~ (x_j, y_j) = (x, y) \} |}{m}.
            \end{align}
        \item    
            Using the estimate $\hat{p}$ of $\Dist$,
            choose a hypothesis that minimizes the
            expected risk, where the expectation is computed according to $\hat{p}$ instead
            of $\Dist$.  I.e., output a hypothesis
            $\hat{h} \in \Hyp$ such that
            \begin{align}
                \hat{h} = \argmin_{h \in \Hyp} \E_{(X, Y) \sim \hat{p}}[\ell(h[X], Y)].
            \end{align}
    \end{enumerate}
    To show that this learning rule is PAC, we note
    that by the strong law of large numbers and the
    fact that $|\X\times \Cl|$ is finite, $m$
    can be chosen sufficiently large so that the total variation distance between $\hat{p}$ and $\Dist$
    is less than $\epsilon$ with probability arbitrarily close to $1$.  This immediately implies that, for any $h \in \Hyp$, simultaneously,
    \begin{align}
        |\E_{(X, Y) \sim \hat{p}}[\ell(h[X], Y)] - R(h)| < \epsilon.
    \end{align}
    This implies the desired property.
With the proofs of Claims 1 and 2 completed, the proof of Theorem~\ref{thm:failure-unif-convergence-pocc} is complete.

\subsection{Proof of Theorem~\ref{thm:failure-unif-convergence-povm}}
\label{proof:failure-unif-convergence-povm}

We first show that $\Hyp$ is compact in the topology generated by the $L_1$ operator norm on operators on $\Hilb$ (here we note that each hypothesis in $\Hyp$ is given by
$(\Pi_0, I - \Pi_0)$, so that $\Hyp$ may be equated with the set of operators $\Pi_0$ yielding the $0$ outcome).
Because the set of operators on $\Hilb$ is finite-dimensional (since $\Hilb$ itself was assumed to be so), all norms on them are equivalent, in the sense that they generate the same topology.
Furthermore, compactness is equivalent to $\Hyp$
being closed and bounded.
We have assumed that
$\Hyp$ is closed.  To show that it is bounded, 
we note that for any $(\Pi_0, I - \Pi_1) \in \Hyp$,
we can write $\Pi_0$ as a convex combination of orthogonal projections:
\begin{align}
    \Pi_0 
    = \sum_{j=1}^{\dim(\Hilb)} a_j \ketbra{v_j}{v_j},
\end{align}
where $a_j \in [0, 1]$.  Then a loose bound on the
$L_1$ operator norm of $\Pi_0$ is given by
\begin{align}
    \| \Pi_0 \|_{op,L_1}
    \leq \sum_{j=1}^{\dim(\Hilb)} a_j
    \leq \dim(\Hilb).
\end{align}
This implies boundedness of $\Hyp$, which implies
that $\Hyp$ is compact in the topology generated by
the $L_1$ operator norm.

Compactness, in turn, implies that $\Hyp$ must have
at least one accumulation point.  Let us assume that
$\Hyp$ has an accumulation point $\Pi_*$ that is
non-deterministic.  Then we may use the argument
from the proof of Theorem~\ref{thm:failure-unif-convergence-pocc} to show that uniform convergence fails and ERM is not PAC.  In particular, there exists an open neighborhood of $\Pi_*$
containing infinitely many non-deterministic elements in $\Hyp$.  This was the key property used in Theorem~\ref{thm:failure-unif-convergence-pocc}, which allows us to conclude that ERM is not PAC for $\Hyp$, and it does not satisfy the uniform convergence property.

\subsection{Proof of Theorem~\ref{thm:finite-dimensional-learnable}}
\label{proof:thm-finite-dimensional-learnable}

The proof of this theorem consists of the following steps:
\begin{enumerate}
    \item 
        We introduce a statistical distance between POVMs -- the total variation
        distance $d_{TV}$ between them.  This distance has the property that any two hypotheses
        within distance $\gamma$ of each other have expected risks within $\gamma$ of
        each other.  Throughout, we choose $\gamma = \epsilon/4$.
    \item
        We show that for every $\gamma$, the $d_{TV}$ $\gamma$-covering number of a
        finite-dimensional POVM class is finite: i.e., it can be covered by finitely
        many $d_{TV}$ balls of radius less than $\gamma$.
    \item
        Using the finiteness of covering numbers, we define the smoothing of the hypothesis
        class by a given $\gamma$-covering, which is a hypothesis class consisting of the
        centers of the balls in the covering.  This class is necessarily finite-cardinality.
    \item 
        By previous results in the literature~\cite{Heidari2021}, the smoothed class is agnostically PAC learnable because it is finite-cardinality.  The output of a
        $(\epsilon/4, \delta)$-PAC learning rule on this hypothesis class has true risk
        within $\epsilon/2$ of the minimum possible within the smoothed class.  This minimum
        has, by our result on the total variation metric, a true risk that is within $\epsilon/2$ of the infimum of possible true risks in the original hypothesis class.
        Thus, the hypothesis returned by the learning rule on the smoothed class has true risk within $\epsilon$ of the infimum for the original class, with probability at least $1-\delta$.
\end{enumerate}

We now give the details of the above steps.

\paragraph{Step 1: Defining $d_{TV}$ between POVMs}
The definition is given in Definition~\ref{def:dtv-povms}.

We next state and prove the lemma connecting $d_{TV}$
with the expected risks of the two hypotheses.
\begin{lemma}[Connecting $d_{TV}$ with expected risks]
    \label{lemma:dtv-risk}
    Let $\Pi_1, \Pi_2 \in \Hyp$.  Then
    \begin{align}
        | R(\Pi_1) - R(\Pi_2) |
        \leq 2d_{TV}(\Pi_1, \Pi_2).
    \end{align}
\end{lemma}
\begin{proof}
    We have
    \begin{align}
        &|R(\Pi_1) - R(\Pi_2)| \\
        &= |\Pr_{(X,Y) \sim \Dist} [ \Out(\Pi_1, X) \neq Y ] -  \Pr_{(X,Y) \sim \Dist} [ \Out(\Pi_2, X) \neq Y ] | \\
        &= \E[ | \Pr[ \Out(\Pi_1, X) \neq Y ~|~ X, Y ] -  \Pr[ \Out(\Pi_2, X) \neq Y ~|~ X, Y ]  | ] \\
        &= \E[ \sum_{b=0}^1 | \Pr[ \Out(\Pi_1, X) = b ~|~ X] -  \Pr[ \Out(\Pi_2, X) = b ~|~ X]  | \cdot \Pr[Y = b] ]  \label{expr:YCondDisappears} \\
        &\leq \E_X[ \sum_{b=0}^1 | \Pr[ \Out(\Pi_1, X) = b ~|~ X] -  \Pr[ \Out(\Pi_2, X) = b ~|~ X]  | ] \\
        &\leq \sup_{x \in \X}  \sum_{b=0}^1 | \Pr[ \Out(\Pi_1, X) = b ~|~ X] -  \Pr[ \Out(\Pi_2, X) = b ~|~ X] \\
        &= 2 d_{TV}(\Pi_1, \Pi_2)
    \end{align}
    
    The conditioning on $Y$ disappears in equation
    (\ref{expr:YCondDisappears}) because the event
    that $\Out(\Pi_j, X) = b$ is independent of $Y$
    given $X$.  The first inequality is by upper 
    bounding $\Pr[Y = b]$ by $1$.
\end{proof}

\paragraph{Step 2: Finiteness of the $d_{TV}$ covering numbers of $\Hyp$}

To show that the $d_{TV}$ covering numbers of $\Hyp$ are finite, it is sufficient to show an upper bound on the corresponding packing numbers.  This is the content
of the next lemma, Lemma~\ref{lemma:finite-total-variation-packing}.

\begin{lemma}[Finiteness of the total variation packing]
    \label{lemma:finite-total-variation-packing}
    Let $\Hyp$ be a finite-dimensional POVM class with outcomes in $\{0, 1\}$.
    For every $\gamma > 0$,  the $\gamma$-packing number of $\Hyp$ is
    finite.
\end{lemma}
\begin{proof}
    Since the Hilbert space $\Hilb$ on which the operators in $\Hyp$ are defined is finite-dimensional, we take an orthonormal basis $\ket{v_1}, ..., \ket{v_d}$ for $\Hilb$.
    We can uniquely encode each $h \in \Hyp$ by its outcome-$0$ operator, which, in turn,
    can be represented by its $d\times d$ matrix over $\Cplx$ with respect to the chosen basis.
    Thus, we can identify $\Hyp$ with a bounded subset $S \subset \Cplx^{d\times d}$ (since
    we already know, from the proof of Theorem~\ref{thm:failure-unif-convergence-povm},
    that finite dimensionality of the POVM class (even if it is not closed) is sufficient
    to conclude boundedness of $\Hyp$ in the $L_{1}$ operator norm).  Bounded subsets of
    finite-dimensional Euclidean spaces are known to have finite packing and covering numbers
    in every norm.  This implies finite packing and covering numbers for $d_{TV}$.
    
\end{proof}

\paragraph{Step 3: Constructing the $TV$-smoothed version of a hypothesis class}

We next define the following \emph{TV-smoothed} hypothesis class.
\begin{definition}[TV-Smoothed version of a POVM class]
    Let $\Hyp$ be a finite-dimensional POVM class with outcomes $0$ and $1$.
    Let $\Part$ be a finite $\gamma$-TV covering of $\Hyp$ by balls with centers
    $\{\Pi^{(j)}\}_{j=1}^{|\Part|}$.  
    We define the $\gamma$-smoothed version of $\Hyp$ to be 
    simply the set of centers.  We denote this by $\Smooth_{TV}(\Hyp, \Part)$.
\end{definition}

\paragraph{Step 4: Learnability of the $TV$-smoothed class implies learnability of the original class}

For any $\gamma$, the $\gamma$-TV-smoothing of a hypothesis class $\Hyp$ has the property 
that its infimum expected loss is less than $\gamma$ away from that of the infimum expected loss of $\Hyp$.  This is the content of the next lemma, Lemma~\ref{lemma:risk-smoothed}.

\begin{lemma}[Infimal expected loss of $\Smooth_{TV}(\Hyp, \Part)$]
    \label{lemma:risk-smoothed}
    Let $\widehat{\Hyp} = \Smooth_{TV}(\Hyp, \Part)$,
    where $\Part$ is a finite $\gamma$-TV covering 
    of $\Hyp$.  Then 
    \begin{align}
        \inf_{h \in \Hyp} R(h)
        \leq \inf_{h \in \widehat{\Hyp}} R(h)
        \leq \inf_{h \in \Hyp} R(h) + 2\gamma.
    \end{align}
\end{lemma}
\begin{proof}
    For any hypothesis $h_* \in \widehat{\Hyp}$ and $h$ in the element of $\Part$ corresponding
    to $h_*$, we have that
    $d_{TV}(h_*, h) \leq \gamma$.  This implies, by Lemma~\ref{lemma:dtv-risk}, that
    \begin{align}
        |R(h_*) - R(h)| \leq 2\gamma,
    \end{align}
    which implies the desired result.
\end{proof}

Lemma~\ref{lemma:risk-smoothed} implies that if
$\widehat{\Hyp}$ is $(\epsilon/2, \delta)$-PAC learnable,
then with probability at least $1-\delta$, 
we can find a hypothesis $h_* \in \widehat{\Hyp} \subseteq \Hyp$ such
that $R(h_*) \leq \inf_{h \in \widehat{\Hyp}} R(h) + \epsilon/2 \leq \inf_{h \in \Hyp} R(h) + \epsilon/2 + 2\gamma = \epsilon$ (recalling that we set $\gamma = \epsilon/4$), which implies that $\Hyp$ is $(\epsilon, \delta)$-PAC learnable.  In fact, since $\widehat{\Hyp}$ is finite,
it is PAC learnable for every $\epsilon, \delta \to 0$, by the results in~\cite{Heidari2021}.  This completes the proof of learnability.  To provide
a more specific sample complexity bound, we recall the following result of
~\cite{Heidari2021}.

\begin{theorem}[\cite{Heidari2021}, Theorem~2]
    \label{thm:mohsen-bound}
    Any finite POVM class $\Hyp$ is agnostic $(\epsilon, \delta)$-PAC-learnable with sample complexity
    bounded by
    \begin{align}
        n_{\Hyp}(\epsilon, \delta)
        = \min_{\Part = (\Part_1, \Part_2, ..., \Part_{|\Part|})} \sum_{r=1}^{|\Part|} \frac{8}{\epsilon^2}\log \frac{2|\Part| |\Part_r|}{\delta},
    \end{align}
    where $\Part$ ranges over all possible joint measurability partitions of $\Hyp$.
\end{theorem}

In our case, we will apply this bound to $\Smooth_{TV}(\Hyp, \Part)$, and hence
to $\Hyp$ itself.  In general, it can be difficult to find a minimal jointly measurable partition,
and so we state the following worst-case bound, which we get from the trivial joint measurability partition whose elements are all singletons.  We let $N$ be the $\gamma$-TV-covering number of $\Hyp$.
\begin{align}
    n_{\Hyp}(\epsilon, \delta)
    \leq n_{\Smooth(\Hyp, \Part)}(\epsilon, \delta)
    \leq \sum_{r=1}^N \frac{8}{\epsilon^2}\log\frac{ 2N }{\delta}
    = \frac{8N}{\epsilon^2} \log\frac{2N}{\delta}.
\end{align}
This completes the proof of Theorem~\ref{thm:finite-dimensional-learnable}.

\subsection{Proof of Theorem~\ref{thm:pac-povm-necessary-and-sufficient}}
\label{proof:thm-pac-povm-necessary-and-sufficient}

\subsubsection{Proof that learnable implies finite fat-shattering dimension}
We first prove the necessary condition.  The chain of logic is as follows:
\begin{enumerate}
    \item
        We first prove that if $\Hyp$ is PAC learnable, then the corresponding
        POCC class is PAC learnable.  We do this by a reduction.
    \item
        We then show that if a POCC class is PAC learnable, then the corresponding
        $p$-concept class is PAC learnable, again by a reduction.
    \item
        It is known that a $p$-concept class is PAC learnable if and only if, for
        every $\gamma$, its $\gamma$-fat shattering dimension is finite.  Since
        all fat-shattering dimensions are the same in our chain of reductions, this
        implies that the $\gamma$-fat shattering dimension of $\Hyp$ must be finite.
\end{enumerate}

\paragraph{Step 1: POVM learnability implies POCC class learnability}

We recall that a POVM class has an associated description as a POCC class.  Namely, every POVM
induces a unique conditional distribution on 
outcomes when used to measure a given state.  

\begin{lemma}
    Suppose that $\Hyp$ is a PAC learnable POVM class.
    Then it is PAC learnable as a POCC class.
\end{lemma}
\begin{proof}
    Let $A$ be an $(\epsilon, \delta)$-PAC learning rule for $\Hyp$.  We may use exactly the same
    learning rule in the POCC framework, and it
    yields the same guarantees.
\end{proof}

\paragraph{Step 2: POCC class learnability implies $p$-concept class learnability}

We recall that every POCC class has an associated $p$-concept class, trivially.

\begin{lemma}
    Suppose that $\Hyp$ is a PAC learnable POCC class.
    Then it is PAC learnable as a $p$-concept class.
\end{lemma}
\begin{proof}
         Let $\Hyp$ be a POCC class, and let $\widehat{\Hyp}$
        denote the corresponding $p$-concept class.
        Note that these have exactly the same
        $\gamma$-fat-shattering dimension, for every
        $\gamma$.  We will show that if $\Hyp$ is
        $(\epsilon, \delta)$-PAC learnable, then
        so is $\widehat{\Hyp}$, with the $L_1$ risk.
        
        Let $A$ be a learning rule for $\Hyp$.
        Then it is also a learning rule for $\widehat{\Hyp}$, and the risks of all hypotheses
        in $\Hyp$ are the same as those in $\widehat{\Hyp}$, as we show next.
        In particular, the misclassification risk for
        $h \in \Hyp$ is given by
        \begin{align}
            R_{\Hyp}(h)
            &= \Pr[ h[X] \neq Y] \\
            &= \Pr[h[X] = 1, Y=0] + \Pr[h[X] = 0, Y=1] \\
            &= \E[ \Pr[ h[X] = 1, Y=0 ~|~ X] + \Pr[h[X]=0, Y=1 ~|~ X]  ] \\
            &= \E[ \Pr[ h[X] = 1~|~X]\Pr[Y=0~|~X] + \Pr[h[X] = 0~|~X]\Pr[Y=1~|~X] ] \\
            &= \E[  h(X) (1 - \Pr[Y=1~|~X]) + (1-h(X))\Pr[Y=1~|~X]    ] \\
            &= \E[ h(X) + \Pr[Y=1~|~X] - 2\Pr[Y=1~|~X]h(X)].
        \end{align}
        Meanwhile, the $L_1$ risk for $h \in \widehat{\Hyp}$ is
        \begin{align}
            R_{\widehat{\Hyp}}(h)
            &= \E[ |h(X) - Y| ] \\
            &= \E[  \E[ |h(X) - Y|    ~|~X]  ] \\
            &= \E[  \Pr[ Y=1 ~|~ X ] 
            (1-h(X)) + (1-\Pr[Y=1~|~X] ) h(X)] \\
            &= \E[ h(X) + \Pr[Y=1~|~X] - 2\Pr[Y=1~|~X]h(X)] \\
            &= R_{\Hyp}(h).
        \end{align}
        This implies that $\Hyp$ being PAC learnable implies that $\widehat{\Hyp}$ is PAC learnable
        with the same parameters.  
\end{proof}

\paragraph{Step 3: $p$-concept class learnability only if fat shattering dimension is finite}
It is shown in ~\cite{Kearns1994} that $p$-concept classes
are learnable with respect to the $L_1$ risk only if 
they have finite $\gamma$-fat-shattering dimension for every $\gamma > 0$.  By steps 1 and 2 of our proof,
this implies that if a POVM class $\Hyp$ is PAC learnable, then it has finite $\gamma$-fat-shattering dimension for every $\gamma$, which completes the proof of the necessary condition of the theorem.

\subsubsection{Proof that finite fat-shattering dimension and finite partitionability implies learnable}

We now show that the stated sufficient conditions imply learnability of a POVM class.
The proof outline when $\Hyp$ is jointly measurable, in which case all that is needed is finite fat shattering dimension, is as follows.
\begin{enumerate}
    \item
        In the preliminaries, 
        we show that the expected value of the denoised empirical risk of a POVM is
        its expected risk (Lemma~\ref{lemma:der-expectation} below).  We also define an appropriate generalization of the
        Rademacher complexity of $\Hyp$.
    \item
        We define $\Phi_{\Hyp}(S)$, for a training set $S$, to be the supremum, over all
        hypotheses $h \in \Hyp$, of the deviation of the denoised empirical risk of $h$
        from the expected risk.  We use McDiarmid's inequality to show that $\Phi_{\Hyp}(S)$
        is well-concentrated around its mean.
    \item
        We upper bound the mean of $\Phi_{\Hyp}(S)$ in terms of the Rademacher complexity of $\Hyp$.
        It is in this step that we use Lemma~\ref{lemma:der-expectation}.
    \item
        We apply known bounds on the Rademacher complexity in terms of covering numbers and
        then a known bound on the covering numbers in terms of the fat shattering dimension.
        This implies an upper bound on the sample complexity of PAC learning $\Hyp$ using
        DERM, which implies learnability for jointly measurable hypothesis classes with finite fat shattering dimension.
\end{enumerate}

To establish that the conditions given in the theorem statement -- namely, that $\Hyp$ has finite fat-shattering dimension and that for every $\alpha > 0$, there exists a finite $\alpha$-approximately jointly measurable partition $\Part$ of $\Hyp$ -- are sufficient, we reason as follows:
\begin{enumerate}
    \item
        We note that the joint measurability smoothing
        of a hypothesis class is a jointly measurable
        class.  Thus, inside the loop of DERM, we are computing the denoised empirical risk $\hat{R}_j$ of a jointly measurable class
        and the denoised empirical risk minimizer
        $\hat{\rho}_j$ of that class.
    \item
        By the reasoning for the jointly measurable
        class case earlier in this proof, with probability at least $1 - \delta/|\Part|$, $\hat{R}_j$ is within $\epsilon/2$
        of the minimal true risk inside the smoothed
        partition element $\hat{P}_j$, provided that
        $|\hat{S}|_{j}$ is chosen to be sufficiently
        large as a function of $\epsilon$, $\delta$, and the number of partition elements $|\Part|$.
        Taking a union bound over all partition elements
        ensures that with probability at least $1-\delta$, this holds for \emph{every} partition element.
    \item
        We show that the $\gamma$-joint measurability smoothing of each partition
        element $P_j$ results in hypotheses whose
        risks are at most $2\gamma$ away from the risks
        of the corresponding unsmoothed hypotheses.
        This is the content of Lemma~\ref{lemma:risk-joint-measurability-smoothed}.  Taking $\gamma = \epsilon/4$, we get
        that the resulting hypothesis $h_*$ returned
        by the learning rule is at most $2\gamma + \epsilon/2 = \epsilon$ away from the infimal risk of the hypothesis class $\Hyp$, with probability at least $1-\delta$.
\end{enumerate}

\paragraph{Step 1: Preliminaries for upper bounding sample complexity for jointly measurable $\Hyp$} %

Our first lemma says that the expected value of the
denoised empirical risk of a POVM hypothesis is the true risk.
\begin{lemma}[Expected value of the denoised empirical risk]
    \label{lemma:der-expectation}
    Let $\Hyp$ be a POVM class, and let $h \in \Hyp$.
    We have the following identity:
    \begin{align}
        \E[ \DER(h, S)]
        = R(h).
    \end{align}
\end{lemma}
\begin{proof}
    This is a result of the tower property of conditional expectation:
    \begin{align}
        \E[\DER(h, S)]
        &= \E[ \E[ \frac{1}{m} \sum_{j=1}^m \ell(h[X_j],   Y_j) ~|~ Y, Z ]  ] 
        = \frac{1}{m} \sum_{j=1}^m \E[ \E[ \ell(h[X_j],  Y_j ~|~ Y_j, Z] ] \\
        &= \frac{1}{m} \sum_{j=1}^m R(h) 
        = R(h).
    \end{align}
\end{proof}

We next define an analogue of the Rademacher complexity~\cite{UnderstandingMachineLearning}
of a POVM class.

\begin{definition}[Rademacher complexity of a jointly measurable POVM class]
    \label{def:rademacher-jointly-measurable}
    Let $\Hyp$ be a POVM class consisting of jointly measurable POVMs with root POVM $\rho$, and let
    $\Dist$ be some distribution on $\X \times \Cl$. 
    
    We define the $m$th Rademacher complexity (with respect to the specified fine graining of which $\rho$ is the root) of $\Hyp$ to be 
    \begin{align}
        \Rad_m(\Hyp)
        = \frac{1}{m} \E_{S,\sigma}[ \sup_{h\in \Hyp} \sum_{j=1}^m \sigma_j \zeta(S_j)]
    \end{align}
    where
    $S \sim \Dist^m$,
    $\zeta(S_j)$ denotes the random variable
    $\E[ \ell(h[X_j], Y_j) ~|~ Y_j, Z_j]$, and
    $Z_j$ is the random outcome of measuring $X_j$
    with $\rho$.
\end{definition}

\paragraph{Step 2: Defining $\Phi_{\Hyp}(S)$ -- the point generalization gap, and applying McDiarmid}

        We define the following function, called the \emph{point generalization gap} for DERM of $\Hyp$ on the dataset $S \in (\X\times\Cl)^m$:
        \begin{align}
            \Phi_H(S)
            = \Phi_{\Hyp}(S)
            = \sup_{h \in \Hyp} ( \DER(h, S) - R(h)).
        \end{align}
        
        Next, we apply McDiarmid's inequality to show that $\Phi(S)$ is close to its expected value with high probability.  To do this, we need to upper bound the
        maximum possible value of $|\Phi(S) - \Phi(\hat{S})|$, where $\hat{S}$ differs from $S$ in exactly one coordinate (say, coordinate $j$).
        We have
        \begin{align}
            &|\Phi(S) - \Phi(\hat{S})| \\
            &= |\sup_{h\in \Hyp} (\DER(h,S) - R(h))
                - \sup_{h\in \Hyp} (\DER(h,\hat{S}) - R(h))| \\
            &\leq |\sup_{h \in \Hyp} (\DER(h, S) - \DER(h, \hat{S}))| \\
            &= \frac{1}{m}|\sup_{h \in \Hyp} \sum_{k \neq j} \E[\ell(h[x_k], y_k) ~|~ (z_k, y_k)] - \E[\ell(h[x_k], y_k) ~|~ (z_k, y_k)] \\
            &~~~~+ \E[\ell(h[x_j], y_j) ~|~ (z_j, y_j)] - \E[\ell(h[\hat{x}_j], \hat{y}_j) ~|~ (\hat{z}_j, \hat{y}_j)] | \\
            &= \frac{1}{m} |\sup_{h\in \Hyp}  |\E[\ell(h[x_j], y_j) ~|~ (z_j, y_j)] - \E[\ell(h[\hat{x}_j], \hat{y}_j) ~|~ (\hat{z}_j, \hat{y}_j)] | \\
            &\leq 1/m.
        \end{align}
        The first inequality is because the difference between suprema is less than or equal to the supremum of the difference.  The final inequality is because the loss is bounded between $0$ and $1$.
        This is essentially exactly the same as in the classical case.
        
        Applying McDiarmid's inequality, we then get the following intermediate bound.
        \begin{lemma}[Application of McDiarmid's inequality to $\Phi(S)$]
            \label{lemma:mcdiarmid-inequality}
            For every $\gamma > 0$, we have
            \begin{align}
                \Pr[ |\Phi(S) - \E_S[\Phi(S)]| \geq \gamma ]
                \leq \exp\left( -2m\gamma^2 \right).
            \end{align}
            Since we want this probability to be at most $\delta$, we choose $\gamma$ such that
            \begin{align}
                e^{-2m\gamma^2} = \delta,
            \end{align}
            which implies
            \begin{align}
                \log \delta = -2m\gamma^2
                \implies \gamma = \sqrt{ \frac{\log(1/\delta)}{2m} }.
            \end{align}
            That is, with probability at least $1-\delta$, we have
            \begin{align}
                \Phi(S) 
                \leq  \E_S[\Phi(S)] +  \sqrt{ \frac{\log(1/\delta)}{2m} }.
            \end{align}
        \end{lemma}

\paragraph{Step 3: Upper bounding $\E[\Phi(S)]$ via Rademacher complexity}
        Next, we need an upper bound on the expected value of $\Phi(S)$ in terms of the Rademacher complexity defined in Definition~\ref{def:rademacher-jointly-measurable}.  
        \begin{lemma}[Rademacher complexity upper bound]
            \label{lemma:rademacher-upper-bound}
            We have the following:
            \begin{align}
                \E[\Phi_{\Hyp}(S)]
                \leq 2\Rad(\Hyp).
            \end{align}
        \end{lemma}
        \begin{proof}
            We follow the pattern as in classical statistical learning theory:
            we first rewrite $R(h)$ in $\E[\Phi(S)]$ as the expected value of an empirical
            expectation.  
            \begin{align}
                \E_S[\Phi(S)]
                = \E[ \sup_{h \in \Hyp} (\DER(h, S) - R(h))  ] 
                = \E[ \sup_{h \in \Hyp} ( \frac{1}{m} \sum_{j=1}^m \E[ \ell(h[X_j], Y_j) ~|~ Y_j, Z_j] - R(h) ) ],
                \label{expr:ExpectedPhi1}
            \end{align}
            where we recall that $Z_j$ is the random
            outcome of measuring $X_j$ with the root
            measurement $\rho$.
            
            We rewrite $R(h)$ as
            \begin{align}
                R(h)
                = \E_{\hat{S}}[ \DER(h, \hat{S})],
            \end{align}
            where $\hat{S}$ is independent and equal in distribution to $S$.  Plugging this into (\ref{expr:ExpectedPhi1}),
            \begin{align}
                \E_S[\Phi(S)]
                = \E_{S,\hat{S}}[ \sup_{h \in \Hyp} ( \frac{1}{m} \sum_{j=1}^m \E[ \ell(h[X_j], Y_j) ~|~ Y_j, Z_j] -  \E[ \ell(h[\hat{X}_j], \hat{X}_j) ~|~ \hat{Y}_j, \hat{Z}_j]) ].
            \end{align}
            The remaining steps can be carried out \emph{exactly} as in the classical case, and we will end up with an upper bound by the Rademacher complexity of the set of denoised empirical risks of hypotheses in the class.  Let $\zeta(S_j)$ denote
            the random variable $\E[\ell(h[X_j], Y_j) ~|~ Y_j, Z_j]$.
            Then the above can be more succinctly written as
            \begin{align}
                &\E_{S,\hat{S}}[ \sup_{h \in \Hyp} ( \frac{1}{m} \sum_{j=1}^m \E[ \ell(h[X_j], Y_j) ~|~ Y_j, Z_j] -  \E[ \ell(h[\hat{X}_j], \hat{Y}_j) ~|~ \hat{Y}_j, \hat{Z}_j]) ] \\
                &= \frac{1}{m} \E_{S,\hat{S}}[ \sup_{h \in \Hyp} \sum_{j=1}^m \zeta(S_j) - \zeta(\hat{S}_j)].
            \end{align}
            Introducing Rademacher random variables $\sigma_j \in \{-1, 1\}$, for $j \in [m]$, this is equal to
            \begin{align}
                \frac{1}{m} \E_{S,\hat{S}}[\sup_{h\in\Hyp} \sum_{j=1}^m \zeta(S_j) - \zeta(\hat{S}_j)]
                = \frac{1}{m} \E_{S,\hat{S},\sigma}[ \sup_{h \in \Hyp} \sum_{j=1}^m \sigma_j \zeta(S_j) - \sigma_j \zeta(\hat{S}_j)    ].
            \end{align}
            Note that this is because if two random variables $\zeta(S_j)$ and $\zeta(\hat{S}_j)$ are independent
            and identically distributed, their difference is equal in distribution to $\sigma_j \cdot (\zeta(S_j) - \zeta(\hat{S}_j))$.
            Finally, this is less than or equal to 
            \begin{align}
                \frac{1}{m} \E_{S,\hat{S},\sigma}[ \sup_{h \in \Hyp} \sum_{j=1}^m \sigma_j \zeta(S_j) - \sigma_j \zeta(\hat{S}_j)    ]
                \leq \frac{2}{m} \E_{S,\hat{S},\sigma}[ \sup_{h\in \Hyp} \sum_{j=1}^m \sigma_j \zeta(S_j)]
                = 2\Rad_m(\Hyp).
            \end{align}
        \end{proof}

\paragraph{Step 4: Upper bounding the Rademacher complexity via covering numbers and fat shattering dimension}
 
         By Step 3, if we want to show that a hypothesis class is learnable, then we should show that
        the Rademacher complexity above is $o(1)$ as 
        $m\to\infty$.  We will show that the 
        \emph{fat-shattering dimension} is the relevant quantity, 
        and that it is sufficient for this quantity to be 
        finite.  Lemmas~\ref{lemma:rademacher-covering}
        and~\ref{lemma:covering-fat-shattering} below establish
        the connection between the Rademacher complexity
        and the fat shattering dimension.

        \begin{definition}[Covering numbers $\CNum(\alpha, \F, m)$]
            Let $\F$ be a metric space with metric $d$.  We say that a subset $\hat{\F} \subseteq \F$
            is an $\alpha$-covering of $\F$ with respect to $d$ if for every $f \in \F$, there exists $g \in \hat{\F}$ such that $d(f, g) \leq \alpha$.  The $\alpha$-covering number of $\F$ with respect to $d$ is defined to be the minimum cardinality of any $\alpha$-covering of $\F$ and is denoted by $\CNum_{d}(\alpha, \F)$.  In place of $d$ in the subscript, we may put a norm, in which case the relevant metric is the one induced by that norm.

            For a class $\F$ of functions $f:\X\to\Y$, 
            we define the following norm %
            with respect to a set $S = (x_1, ..., x_m) \in \X^m$:
            \begin{align}
                \|f - g\|_{\infty,S}
                = \sup_{j \in \{1, 2, ..., m\}} |f(x_j) - g(x_j)|.
            \end{align}
            The covering numbers of $\F$ with respect to this norm on the set $S$
            are well-defined.
            We then define the covering number
            $\CNum(\alpha, \F, m) = \sup_{S \in \X^m} \CNum_{\|\cdot\|_{\infty,S}}(\alpha, \F, S)$.  %
        \end{definition}

        \begin{lemma}[Upper bound on Rademacher complexity via covering numbers~\cite{ChervonenkisFestschrift}]
            \label{lemma:rademacher-covering}
            Let $\F$ be a function class.  Then
            \begin{align}
                \Rad_m(\F)
                \leq \inf_{\alpha} \left(\alpha 
                    + \sqrt{\frac{2\log(\CNum(\alpha, \F, m))}{m}  } \right).
            \end{align}
        \end{lemma}

        \begin{lemma}[Upper bound on covering numbers via fat-shattering dimension~\cite{ChervonenkisFestschrift}]
            \label{lemma:covering-fat-shattering}
            Let $\F$ be a class of functions with range $[0, 1]$.  Then the following upper bound on the covering numbers holds: let $\alpha \geq 0$ and $d = \fat_{\alpha/4}(\F)$.  Then
            \begin{align}
                \CNum(\alpha, \F, m)
                \leq 2 \cdot \left( m \cdot (2/\alpha + 1)^2\right)^{\ceil{ d \log\left(\frac{2em}{d\alpha}\right) }}
            \end{align}
        \end{lemma}
    
    Plugging the upper bound in Lemma~\ref{lemma:covering-fat-shattering} into
    the one in Lemma~\ref{lemma:rademacher-covering},
    and then plugging that into Lemma~\ref{lemma:rademacher-upper-bound} and, finally, using Lemma~\ref{lemma:mcdiarmid-inequality} yields the following statement:
    with probability at least $1-\delta$,
    \begin{align}
        \Phi(S)
        \leq 2\Rad_m(\Hyp)
        \leq 2\inf_{\alpha}\left( \alpha +
            \sqrt{ 
                \frac{2\log 2 + 2 \ceil{ d \log\left(\frac{2em}{d\alpha}\right)} \log(m(2/\alpha + 1 )^2)}
                {m}
            } \right),
    \end{align}
    which can be upper bounded by setting $\alpha = \Theta(\frac{\log m}{\sqrt{m}})$.
    This results in the following bound:
    \begin{align}
        \E_{S}[\Phi(S)]
        \leq O(\frac{\sqrt{d}\log m}{\sqrt{m}})
        \implies
        \Phi(S)
        \leq O(\frac{\sqrt{d}\log m + \sqrt{\log(1/\delta)}}{\sqrt{m}})
    \end{align}
    
    This directly translates to a finite upper bound
    on the number of samples $n_{\Hyp}(\epsilon, \delta)$ required to achieve
    a risk within $\epsilon$ of the infimum, provided
    that the fat shattering dimension $d$ is finite.  Specifically,
    \begin{align}
        n_{\Hyp}(\epsilon, \delta)
        \leq
        \inf_{\Part}  O(|\Part|\frac{d + \log(1/\delta)}{\epsilon^2}),
    \end{align}
    where $\Part$ ranges over all finite $\epsilon/8$-almost-jointly-measurable partitions of $\Hyp$.
    
    This completes the proof of sufficiency of finite fat shattering dimension for learnability in the case where $\Hyp$ is a jointly measurable class.

\paragraph{Completing the proof of the theorem}
It remains to prove the following lemma, which is the remaining detail for completing the argument in the case where $\Hyp$ can be partitioned into finitely many approximately jointly measurable classes:

\begin{lemma}[Risk of joint measurability-smoothed hypothesis classes]
    \label{lemma:risk-joint-measurability-smoothed}
    Let $\hat{\Hyp} = \Smooth_{JM}(\Hyp, \Part)$,
    where $\Part$ is a finite $\gamma$-joint measurability smoothing of $\Hyp$.  Then
    \begin{align}
        \inf_{h \in \Hyp} R(h) - 2\gamma \leq \inf_{h \in \hat{\Hyp}} R(h) \leq \inf_{h \in \Hyp} R(h) + 2\gamma.
    \end{align}
\end{lemma}
\begin{proof}
    It suffices to prove the analogous chain of inequalities for any single hypothesis $h$ and its
    smoothed version $\hat{h}$.  For this, we use
    the data processing inequality for total variation distance: let $(\Pi, \alpha)$ and $(\hat{\Pi}, \alpha)$ be respective fine grainings of
    $h$ and $\hat{h}$, noting that by definition of
    the smoothing operation that the two classical channels are the same.  Then for any $x \in \X$,
    \begin{align}
        d_{TV}(\alpha\circ \Out(\Pi, x), \alpha \circ\Out(\hat{\Pi}, x))
        \leq d_{TV}(\Out(\Pi, x), \Out(\hat{\Pi}, x))
        \leq \gamma.
    \end{align}
    This implies, by Lemma~\ref{lemma:dtv-risk},
    that $|R(h) - R(\hat{h})| \leq 2\gamma$, which implies the stated result.
\end{proof}

This completes the proof of the sufficient condition
part of Theorem~\ref{thm:pac-povm-necessary-and-sufficient}.

\subsection{Proof of Theorem~\ref{thm:fat-shattering-joint-measurability}}
\label{proof:thm-fat-shattering-joint-measurability}

\begin{lemma}[Relating $d_{TV}$ covering and packing numbers to $\JMCovNumber(\gamma, S)$]
    \label{lemma:dTV-to-jm-covering}
    Let $S$ be a collection of POVMs
    We have 
    \begin{align}
        \JMCovNumber(\gamma, S) \leq \CNum(\gamma, S, d_{TV}) \leq \PNum(\gamma, S, d_{TV}).
    \end{align}
\end{lemma}
\begin{proof}
    The second inequality is well-known, so we focus on the first.  Let
    $\hat{S}$ be a $\gamma$-covering of $S$.  We will show that $\hat{S}$ is also
    a $\gamma$-jm covering of $S$, which immediately implies the inequality.

    Let $\Pi$ be the center of one element of $\hat{S}$.  We claim that the closed $d_{TV}$ ball $B_{TV}(\Pi, \gamma)$ centered at $\Pi$ with radius $\gamma$ is $\gamma$-jointly measurable.  We choose the root POVM to be $\Pi$ itself.  Now, for every $\Pi' \in B_{TV}(\Pi, \gamma)$, we consider the trivial fine-graining $(\Pi', Id)$, where $Id$ is the identity channel.  Trivially, $d_{TV}(\Pi, \Pi') \leq \gamma$.  This completes the proof of the claim.
\end{proof}

The next lemma relates the $d_{TV}$-packing number of a set of POVMs to the $m$-sample $d_{TV}$-covering number, which we recall is upper bounded by a function of the fat-shattering dimension.

\begin{lemma}[Relating the $d_{TV}$-packing number to the $m$-sample $d_{TV}$-covering number]
    \label{lemma:packing-to-m-sample-covering}
    Let $\Hyp$ be a class of POVMs.
    Suppose that $k$ is some number satisfying
    $k \leq \PNum(\gamma, \Hyp, d_{TV})$.
    If $m \geq {\PNum(\gamma, \Hyp, d_{TV})\choose 2}$, then
    \begin{align}
        \PNum(\gamma, \Hyp, d_{TV})
        \leq \CNum(\gamma, \Hyp, m).
    \end{align}
\end{lemma}
\begin{proof}
    To lower bound the $m$-sample covering number of $\Hyp$ by some number $k$, we must exhibit a set $S$ of $m$ points $\{\hat{\rho}_j\}_{j=1}^m$ such that $\CNum(\gamma, \Hyp, S) \geq k$.

    An important consequence of Lemma~\ref{lemma:dTV-to-jm-covering} is that if there exists a $\gamma$-jm partition of $\Hyp$ with cardinality at least $k$, then there exist hypotheses $h_1, ..., h_k \in \Hyp$ such that for every $i \neq j \in [k]$, we have $d_{TV}(h_i, h_j) \geq \gamma$.  This is a direct consequence of the packing number bound in the lemma.

    The fact that $d_{TV}(h_i, h_j) \geq \gamma$ means that there exists a state
    $\rho_{i,j} \in \X$ such that $d_{TV}(h_i(\rho_{i,j}), h_j(\rho_{i,j})) \geq \gamma$.  Thus, we choose our set $S$ to be $\{\rho_{i,j}\}_{i\neq j \in [k]}$, along with an arbitrary collection of $m-{k\choose 2}$ other states.  It is then easily checked that the covering number $\CNum(\gamma, \Hyp, S) \geq k$,
    and we can set $k = \PNum(\gamma, \Hyp, d_{TV})$.  This completes the proof.
\end{proof}

The proof of the theorem is then a direct application of Lemmas~\ref{lemma:dTV-to-jm-covering}, ~\ref{lemma:packing-to-m-sample-covering}, and ~\ref{lemma:covering-fat-shattering}.

\section{Definitions from quantum mechanics}
\label{sec:quantum-for-learning}

We give below a brief introduction to relevant definitions and notation from quantum information.
This is meant only to highlight the bare minimum necessary concepts for this paper.  The reader is encouraged
to consult~\cite{BkWilde_2017} for more extensive discussions of quantum information.  We note that the reader does not need any physics background at all, and the required mathematics is not beyond the training of most learning theorists.

To describe quantum states, we fix a Hilbert space $\Hilb$ over the complex numbers $\Cplx$.  A \emph{pure state} is a unit vector
in $\Hilb$, which, in the bra-ket notation of quantum mechanics, is denoted by
$\ket{v}$.  The dual space to $\Hilb$ is the vector space of linear functionals $\bra{v}:\Hilb\to \Cplx$,
where $\bra{v}\ket{w}$ is defined to be the inner product of $\ket{v}$ with $\ket{w}$.
It is frequently convenient to identify pure states $\ket{v}$ with their outer product
forms $\ketbra{v}{v}$, which are operators from $\Hilb\to\Hilb$.  The reason for this is the outer product forms fit nicely into the density matrix formalism, which we discuss next.

Mixed states (which we just call states in this paper) are more general: they are convex combinations of pure states and are also called \emph{density matrices}.  They capture statistics resulting from drawing pure states
from a probability distribution.  However, it should be noted that a single density matrix
can arise from multiple distinct convex combinations of pure states.

A quantum measurement is specified by a positive operator-valued measure (POVM), defined as follows.

\begin{definition}[POVM]
    \label{def:povm}
    A POVM with $k$ outcomes, defined on a Hilbert space $\Hilb$, is a $k$-tuple $\Pi = (\Pi_1, ..., \Pi_k)$ of 
    positive semidefinite Hermitian operators on $\Hilb$ that sum to the identity operator.
\end{definition}
We note that each operator $\Pi_j$, by virtue of being positive semidefinite Hermitian,
has a unique decomposition as $\Pi_j = M_j^{*}M_j$, where $M_j^*$ denotes the adjoint operator.

Measurement of a mixed state $\rho$ by a POVM $\Pi$ works as follows: it produces an \emph{outcome} $\Out(\Pi, \rho)$ in $\{1, ..., k\}$, which is observed by the measurer and a post-measurement state $\rho'$, which is not.  The outcome is drawn from the following distribution:
\begin{align}
    \Pr[ \Out(\Pi, \rho) = j] = \Tr{\rho \Pi_j},
\end{align}
where $\Tr{\cdot}$ denotes the trace.

The post-measurement state $\rho'$ is dependent on $\Out(\Pi, \rho)$.  If the outcome is $j$, then
$\rho'$ is given by
\begin{align}
    \rho' = \frac{M_j \rho M_j^*}{\Tr{\Pi_j \rho}}.
\end{align}

These measurement rules are collectively called the \emph{Born rule}.  One can thus think of a POVM as a particular type of stochastic map from density matrices to ordered pairs whose first component is an outcome index and whose second component is a post-measurement density matrix.  The particulars of the Born rule become important when one tries to define specific hypothesis classes and study their learning-theoretic measures of complexity (e.g., fat-shattering dimension).  It is also worth emphasizing a few phenomena that differentiate the quantum learning setting from the classical case:
\begin{itemize}
    \item 
        Unknown states cannot be copied.  That is, there is no general procedure that takes as input
        a register prepared in some state $\rho$ and produces two registers, both in state $\rho$.
        Thus, for example, a learner cannot make a ``backup copy'' of a state in the training set.
    \item
        States cannot be directly observed by a learner.  The only thing that can be observed is the outcome index of measurement of a state.
\end{itemize}

\section{Aspects of learning theory for those only familiar with quantum information}    
\label{sec:learning-for-quantum}

Here we describe the basics of classical statistical learning theory for an audience that may not be familiar with it.  Our goal is to avoid common confusions, such as the distinction between state estimation and learning.  This distinction is important in Section~\ref{sec:prior-work-discussion}.

In classical statistical learning theory, supervised learning is formulated as follows: a domain $\X$
and a co-domain $\Y$ (which we think of as the label set in a classification problem) are fixed and known to the learner.  There is an unknown joint distribution $\Dist$ on $\X\times \Y$.  A known hypothesis class $\Hyp$ consisting of deterministic functions $h:\X\to\Y$ is fixed.  These hypotheses are meant to approximate the statistical association between inputs $x \in \X$ and labels $y\in \Y$.  To measure the quality of approximation, a \emph{loss function} $\ell:\Y\times \Y$ is fixed, and the loss of a hypothesis on a pair in $\X\times \Y$ is defined by $\ell(h, x, y) = \ell(h(x), y)$.  The \emph{risk} of a hypothesis is its expected loss on a pair $(X, Y) \sim \Dist$: $R(h) = \E_{(X, Y) \sim \Dist}[\ell(h, X, Y)]$.

The learner sees a training set consisting of independent samples from $\Dist$, and the goal of the learner is to choose a hypothesis $h$ from $\Hyp$ with risk as close as possible to the worst risk of any hypothesis in the class.  Formally, a hypothesis class is $(\epsilon, \delta)$-PAC learnable if there exists a learning rule (i.e., a function from datasets to $\Hyp$) $A$ and a number of samples $m(\epsilon, \delta)$ such that, for \emph{every} distribution $\Dist$, $A$ outputs a hypothesis $h$ such that with probability at least $1-\delta$,
\begin{align}
    \label{expr:risk-bound}
    R(h) \leq \inf_{h_* \in \Hyp} R(h_*) + \epsilon.
\end{align}
Note that $\Hyp$ may be uncountably infinite, and so the infimum may not be achievable.  The number of samples required for the risk bound (\ref{expr:risk-bound}) to hold with probability $\geq 1-\delta$ is the \emph{sample complexity} of learning the hypothesis class. 

We emphasize a few things about this:
\begin{itemize}
    \item
        The above framework is \emph{distribution-free}, in the sense that the number of samples and the learner must not depend a priori on any assumption about the form that $\Dist$ takes.  However, the learner is assumed to have full knowledge of $\X$, $\Y$, $\ell(\cdot, \cdot, \cdot)$, and $\Hyp$.
    \item
        Statistical learning theory does not deal with computational efficiency, as learning rules are not algorithms.  Indeed, a hypothesis class may be PAC learnable but not efficiently so.
    \item
        The goal is to choose a hypothesis that captures the statistical association between $X$ and $Y$ as well as possible \emph{compared to any other hypothesis in the class}.  This is a distinct approach from estimating the distribution $\Dist$ of the data.  The reason that the theory is formulated this way is that the problem of selecting a hypothesis from a well-designed class $\Hyp$ can have dramatically smaller sample complexity than that of estimating $\Dist$.  This is a key difference between PAC learning and estimation.
\end{itemize}
Further philosophical grounding for statistical learning theory can be found in any of a number of textbooks on the subject (e.g.,~\cite{UnderstandingMachineLearning}).

\subsection{Empirical risk minimization}
The quintessential learning rule in classical learning theory is \emph{empirical risk minimization}.  Given a dataset $S = ((X_j, Y_j))_{j=1}^m$ and a hypothesis $h \in \Hyp$, the empirical risk of $h$ is given by
\begin{align}
    \hat{R}(h, S)
    = \frac{1}{m}\sum_{j=1}^m \ell(h(X_j), Y_j).
\end{align}
Then the empirical risk minimization (ERM) learning rule outputs the following:
\begin{align}
    h_* = \argmin_{h \in \Hyp} \hat{R}(h, S).
\end{align}
This is a central learning rule in the classical theory, as explained in Section~\ref{sec:fundamental-theorem}.

\subsection{The fundamental theorem of concept learning}
\label{sec:fundamental-theorem}

One of the fundamental results in classical statistical learning theory is the \emph{fundamental theorem of concept learning},
sometimes called the fundamental theorem of PAC learning or of statistical learning (see~\cite{UnderstandingMachineLearning}, Theorem 6.7).  It gives matching necessary and sufficient conditions for a hypothesis class to be learnable, under certain assumptions on the codomain $\Y$ and the loss function.  Specifically, there is a combinatorial notion of complexity of the hypothesis class, known as the Vapnik-Chervonenkis (VC) dimension of $\Hyp$.  The fundamental theorem of concept learning relates the VC dimension of $\Hyp$ to its learnability.  We summarize it below.

\begin{theorem}[Fundamental theorem of concept learning~\cite{UnderstandingMachineLearning}]
    Let $\Hyp$ be a hypothesis class of functions from a domain $\X$ to $\{0, 1\}$,
    and let the loss function be the misclassification loss.  Then the following are equivalent:
    \begin{enumerate}
        \item 
            The ERM rule is a successful PAC learner for $\Hyp$.
        \item 
            $\Hyp$ is PAC learnable.
        \item    
            $\Hyp$ has finite VC dimension.
    \end{enumerate}
\end{theorem}

\section{Further discussion of prior work and relationship to tomography problems}
\label{sec:prior-work-discussion}

Here we contrast our work with works on state and channel/process tomography.  Our main messages are as follows:
\begin{enumerate}
    \item
        It is not obvious how channel or state tomography can be used to construct a learning rule;
    \item
        In particular, none of our results seem to follow from PAC learning results for state or channel tomography.
\end{enumerate}
Our motivation for emphasizing these points is that readers sometimes confuse PAC frameworks for estimation/tomography with PAC learning, despite the fact that these are distinct problems.

We start by defining the basic versions of both state and channel tomography.

In state tomography, the problem is as follows: given $m$ copies of an unknown mixed state $\rho$,
produce an estimate $\hat{\rho}$ such that with probability $> 1-\delta$, $\|\hat{\rho} - \rho\|_F^2 \leq \epsilon$,
where $\|\cdot\|_{F}$ is the Frobenius norm. 

In channel tomography, one is given access to a quantum channel (formalized by a completely positive trace-preserving (CPTP) map) $\Phi$, and one is allowed to query it $m$ times by preparing input states, passing them through the channel, and measuring the output.  The goal is to produce an estimate
$\hat{\Phi}$ that is within $\epsilon$ of $\Phi$ in some metric, with probability $> 1-\delta$. 

In both state and channel tomography, it is known that in the finite-dimensional case, only finitely many samples are needed.  One may be tempted to try to use a solution to either problem to perform PAC learning with respect to a hypothesis class $\Hyp$.  We give a few examples to show the flaws in such approaches.

\begin{enumerate}
    \item
        One can imagine viewing the POVM that we want to learn as a quantum channel and using channel
        tomography to estimate it.  There are multiple problems with this approach: the result of channel tomography need not be a POVM in the hypothesis class.  More seriously, channel tomography requires that we be allowed to prepare registers in arbitrary states (which we know) and feed them into the channel.  Such power is not given to the learner in the PAC learning setting: in fact, input states are drawn from an unknown distribution, and they
        are not known to us.
        
    \item
        One might suppose that if we give a little bit more power to channel tomography, then it might become relevant to PAC learning.  In particular,
        suppose that, by any method whatsoever, we could produce a POVM $\Pi$ (not necessarily in the hypothesis class!) such that $R(\Pi) \leq \epsilon + \inf_{\Pi_*} R(\Pi_*)$, where the infimum is taken over all POVMs.
        A learning rule that uses this ability must still produce an $h_*$ that lies in the hypothesis class $\Hyp$ and is within $\epsilon$ of $\inf_{h \in \mathcal{H}} R(h)$ with probability at least $1-\delta$.
        One natural idea would be to choose $h_* = \argmin_{h \in \Hyp} d_{TV}(h, \Pi)$.  But this does not solve the problem: there exist learning scenarios in which there are multiple,     well-separated $\Pi_*$ that achieve nearly the minimum possible risk (which is not necessarily $0$, since the distribution on input states may place positive probability on states that are close together) over all possible POVMs (not just the ones in the hypothesis class).  In this case, channel tomography may output a $\Pi$ that is far from every hypothesis in the class, while there may exist a hypothesis $h \in \mathcal{H}$ that is very close to some \emph{other} POVM $\Pi_*$ with nearly minimal risk.  This would violate the agnostic PAC learning condition.  
        
    \item
        One might instead think to use state tomography.  In particular, the Choi-Jamio\l{}kowski isomorphism result states that, given a CPTP map, if we input a suitably defined maximally mixed state in a larger space and have access to the extended CPTP map that acts via the identity on the environment, then the resulting state completely characterizes the CPTP map.
        We should note that the learner cannot construct this output state, because the learner cannot provide arbitrary inputs to the CPTP map (the inputs are decided strictly by the data-generating distribution $\Dist$, not the learner, as is the situation in classical statistical learning theory).  Thus, quantum state tomography \textbf{cannot} be brought to bear to recover this state, and so we cannot even estimate it, let alone attack the learning
        problem.
\end{enumerate}

\section{A formalization of quantum learning rules}
\label{sec:pac-formalization}

There is some mathematical subtlety in defining a learning rule in the quantum
setting.  In particular, what it means, informally, for a learning rule to only
be able to interact with a quantum register by measurement may be clear, but
the formalization in terms of mathematical objects is less straightforward.  For completeness, we give such a formalization in this section, in terms of Markov decision processes.

We define the following Markov decision process for a given dataset $S = ( (X_1, Y_1), ..., (X_m, Y_m))$: the \emph{state} $Z_0$ is initialized to $\tensor_{j=1}^m X_j$.  At any timestep $t$, the set of possible actions consists of POVMs operating on the state $Z_t$, producing an observable outcome $\omega_t$ via the Born rule.  The state $Z_{t+1}$ is then derived from $S_j$ again via the Born rule. 
    
A POVM learning rule specifies a policy for this MDP, where, at each timestep $t \geq 0$, the action $A_t$ at time $t$ is conditionally independent of $Z_{j}$ for any $j$, given the outcomes $\omega_0, ..., \omega_{t-1}$.  Finally, the learning rule specifies a conditional distribution from outcome sequences to hypotheses $h \in \Hyp$.

\section{More details on sample complexity bounds for variational quantum circuits}
\label{sec:qnn-discussion}
Here we give more details for our sample complexity bounds for hypothesis classes corresponding to variational quantum circuits.

In great generality, one can define the following hypothesis class.

\begin{definition}[General variational quantum circuit]
    \label{def:general-variational}
    Let $\Hilb$ be a Hilbert space with dimension $d$.  We fix a POVM $\Pi$
    on $\Hilb$ with two outcomes.  We define the hypothesis class
    $\GeneralVar$ to consist of hypotheses parametrized by an arbitrary unitary
    operator $U$ from $\Hilb$ to $\Hilb$ that apply $U$ to the input state
    $\rho$, then measure the resulting state with $\Pi$.
\end{definition}

We have the following theorem giving the sample complexity of $\GeneralVar$.
\begin{theorem}[Sample complexity bound for $\GeneralVar$]
    The class $\GeneralVar$ is $(\epsilon, \delta)$-PAC learnable with sample complexity
    \begin{align}
        n_{\GeneralVar}(\epsilon, \delta)
        \leq 
        O(  \frac{8N}{\epsilon^2}\log\frac{2N}{\delta})
        \leq (C/\epsilon)^{d+2} \log(1/\delta).
    \end{align}
\end{theorem}
\begin{proof}
    By Theorem~\ref{thm:finite-dimensional-learnable}, it is sufficient to upper bound the
    $\epsilon/4$-TV-covering number of $\GeneralVar$.  We first note that for any $\gamma > 0$,
    the $\gamma$-TV covering number of $\GeneralVar$ is upper bounded by the $\gamma$-$L_1$-operator norm covering
    number of the set of $d\times d$ unitary matrices, which in turn is upper bounded by the $\gamma$-$L_1$ norm covering number of the set of $d\times d$ unitary matrices, viewed as vectors.

    The $\gamma$-$L_1$-operator norm covering number of the set of $d\times d$ unitary matrices
    is
    \begin{align}
        \leq d(C/\gamma)^{d},
    \end{align}
    for some positive constant $C$, by upper bounding by the $\gamma$-$L_2$-operator norm covering number and using Theorem~7 of~\cite{SzarekCoveringNumbers}.
    
    This implies that $N \leq d(C/\gamma)^d$, and setting $\gamma = \epsilon/4$, we get
    \begin{align}
        N \leq d(C/\epsilon)^d,
    \end{align}
    resulting in the bound
    \begin{align}
        n(\epsilon, \delta)
        \leq O(  \frac{d(C/\epsilon)^d}{\epsilon^2}\log\frac{2d(C/\epsilon)^d}{\delta}).
    \end{align}
    Upper bounding factors that are polynomial in $d$ and $\epsilon$ by $(C/\epsilon)^d$, this can be simplified to
    \begin{align}
        n(\epsilon, \delta)
        \leq (C/\epsilon)^{d+2} \log(1/\delta),
    \end{align}
    which completes the proof.
\end{proof}

This sample complexity scales exponentially with the dimension -- as one might expect, since the hypotheses may be thought of as applying an arbitrary circuit to the input, then measuring by a fixed POVM, and so such a hypothesis class is
not used in practice.  Our purpose in spelling out this example is to illustrate the type of analysis that one would undertake in applying our results to a specific hypothesis class.
Instead, the approach taken in works on variational quantum circuits is to
constrain the form of the unitary operator.  Specifically, ~\cite{MohsenQNN} constrains $U$ to
be \emph{bandlimited} in terms of its Fourier spectrum, which reduces the number of trainable parameters and, hence, the TV-covering number of the hypothesis space to a polynomial value in the dimension.

\end{document}